\documentclass[twoside,leqno,twocolumn]{article}
\usepackage{ltexpprt}

\makeatletter
\let\proof\@undefined
\let\endproof\@undefined
\let\theorem\@undefined
\let\lemma\@undefined
\let\proposition\@undefined
\let\fact\@undefined
\let\algorithm\@undefined
\makeatother

\renewcommand{\leq}{\leqslant}
\renewcommand{\ge}{\geqslant}
\renewcommand{\geq}{\geqslant}
\usepackage{amsmath,amssymb,amsfonts,amsthm,epsfig,bm,xspace,epsfig,latexsym}

\newtheorem{theorem}{Theorem}[section]
\newtheorem{lemma}[theorem]{Lemma}
\newtheorem{claim}[theorem]{Claim}
\newtheorem{proposition}[theorem]{Proposition}

\newtheorem{fact}[theorem]{Fact}
\newtheorem{definition}[theorem]{Definition}
\newtheorem{remark}[theorem]{Remark}
\newtheorem{conjecture}[theorem]{Conjecture}

\newtheorem{algorithm}[theorem]{Algorithm}
\newcommand{\ignore}[1]{}

\newcommand{\rnote}[1]{\footnote{{\bf [[Rocco: {#1}\bf ]] }}}

\newcommand{\ynote}[1]{[{\tiny Yi: \bf #1}]\marginpar{*}}
\renewcommand{\Pr}{{\bf Pr}}

\newcommand{\E}{{\bf E}}
\newcommand{\Ex}{\mathop{\bf E\/}}

\newcommand{\Fold}{\mathrm{Fold}}
\newcommand{\R}{\mathbb R}
\newcommand{\N}{\mathbb N}

\newcommand{\bn}{\bits^n}

\newcommand{\eps}{\epsilon}

\newcommand{\sign}{\mathrm{sign}}
\newcommand{\sgn}{\mathrm{sign}}

\newcommand{\weight}{\mathrm{wt}}

\newcommand{\littlesum}{{\textstyle \sum}}
\newcommand{\littleprod}{{\textstyle \prod}}

\newcommand{\half}{{\textstyle \frac12}}

\newcommand{\calL}{{\cal L}}
\newcommand{\calD}{{\cal D}}

\newcommand{\calH}{{\cal H}}

\newcommand{\calT}{{\cal T}}
\newcommand{\calC}{{\cal C}}

\newcommand{\calG}{{\cal G}}

\newcommand{\RP}{\mathrm{RP}}
\newcommand{\NP}{\mathrm{NP}}
\newcommand{\Pp}{\mathrm{P}}

\newcommand{\opt}{\mathrm{Opt}}
\newcommand{\Opt}{\opt}

\newcommand{\mcl}[1]{\mathcal {#1}}

\newcommand{\cD}{\mcl{D}}

\newcommand{\bits}{\{0,1\}}
\newcommand{\vnote}[1]{}
\newcommand{\pnote}[1]{}
\newcommand{\yinote}[1]{}
\newcommand{\vinote}[1]{}

\begin{document}

\title{Hardness Results for Agnostically Learning Low-Degree Polynomial
Threshold Functions}

\author{Ilias Diakonikolas\thanks{Research supported by NSF grants CCF-0728736, CCF-0525260,
and by an Alexander S. Onassis Foundation Fellowship.}\\
Columbia University\\
{\tt ilias@cs.columbia.edu}\\
\and
Ryan O'Donnell\thanks{Supported by NSF grants CCF-0747250 and CCF-0915893, BSF grant
2008477, and Sloan and Okawa fellowships.}\\
CMU\\
{\tt odonnell@cs.cmu.edu}\\
\and
Rocco A. Servedio\thanks{Supported by NSF grants
CCF-0347282, CCF-0523664 and CNS-0716245, and
by DARPA award HR0011-08-1-0069.}\\
Columbia University\\
{\tt rocco@cs.columbia.edu}\\
\and
Yi Wu\thanks{Part of the work is done when the author is at CMU,  supported by the National Science Foundation under grant numbers CCF-0747250, CCR-0122588; US Army Research Office under grant number DAAD-190210389; and generous support from International Business Machines.}
\\
IBM Almaden\\
{\tt wuyi@us.ibm.com}
}
\date{}

\maketitle


\begin{abstract} \small\baselineskip=9pt 
Hardness results for maximum agreement problems have close connections to hardness results for
proper learning in computational learning theory.  In this paper we prove two hardness results for the problem of finding a low degree polynomial threshold function (PTF) which has the maximum possible agreement with a given set of labeled examples in $\R^n \times \{-1,1\}.$
We prove that for any constants $d\geq 1, \eps > 0$,

\begin{itemize}

\item
Assuming the Unique Games Conjecture, no polynomial-time algorithm can find a degree-$d$ PTF that
is consistent with a $(\half + \eps)$ fraction of a given set of labeled examples in $\R^n \times \{-1,1\}$, even if there exists a degree-$d$ PTF that is consistent with a $1-\eps$ fraction of the examples.

\item It is $\NP$-hard to find a degree-2 PTF that is consistent with
a $(\half + \eps)$ fraction of a given set of labeled examples in $\R^n \times \{-1,1\}$, even if
there exists a halfspace (degree-1 PTF) that is consistent with a $1 - \eps$ fraction of the
examples.

\end{itemize}

These results immediately imply the following hardness of learning results: (i) Assuming the
Unique Games Conjecture, there is no better-than-trivial proper  learning algorithm that agnostically learns degree-$d$ PTFs under arbitrary distributions; (ii)  There is no better-than-trivial learning algorithm  that outputs degree-$2$ PTFs and agnostically learns halfspaces (i.e. degree-$1$ PTFs) under arbitrary distributions.
\end{abstract}

\section{Introduction}
A \emph{polynomial threshold function} (PTF) of degree $d$ is a function
$f: \R^n \to \{-1,+1\}$ of the form $f(x) = \sign(p(x))$, where
\[p(x)=\sum_{\text{multiset } S\subseteq[n], |S|\leq d} c_{S} \prod_{i\in S} x_i\]
is a degree-$d$ multivariate
polynomial with real coefficients. Degree-$1$  PTFs are commonly known as \emph{halfspaces} or \emph{linear threshold functions}, and have been intensively studied for decades in fields as diverse as theoretical neuroscience, social choice theory and Boolean circuit complexity.

The last few years have witnessed a surge of research interest and results in theoretical computer science on halfspaces and low-degree PTFs, see e.g. \cite{Servedio:07cc,RS:08,FGKP:journal,FGRW09,GR09journal,DHK+:10, Kane:10}.
One reason for this interest is the central role played by low-degree PTFs (and halfspaces in particular) in both practical and theoretical aspects of \emph{machine learning}, where many learning algorithms either implicitly or explicitly use low-degree PTFs as their hypotheses.  More specifically, several widely used linear separator learning algorithms such as the Perceptron algorithm and the ``maximum margin'' algorithm at the heart of Support Vector Machines output halfspaces as their hypotheses.  These and other halfspace-based learning methods are commonly augmented in practice with the ``kernel trick,'' which makes it possible to efficiently run these algorithms over an expanded feature space and thus potentially learn from labeled  data that is not linearly separable in $\R^n$.   The ``polynomial kernel'' is a popular kernel to use in this way; when, as is usually the case, the degree parameter in the polynomial kernel is set to be a small constant, these algorithms output hypotheses that are equivalent to low-degree PTFs.  Low-degree PTFs are also used as hypotheses in several important learning algorithms with a more complexity-theoretic flavor, such as the low-degree algorithm of Linial \emph{et al.} \cite{LMN93} and its variants \cite{Jackson:97,OdonnellServedio:07}, including some algorithms for distribution-specific agnostic learning \cite{KKMS05,KOS:08,BOW:08,DHK+:10}.

\ignore{
Algorithms that learn using low-degree PTF hypotheses play an important role in computational learning theory as well as in practical machine learning.  The seminal work of Linial, Mansour and Nisan ~\cite{LMN93} proposed the ``low degree algorithm'' that works by estimating all the low-degree Fourier coefficients of the target function and outputting the sign of the resulting low-degree polynomial approximator to the target function.   This algorithm and its variants have proved immensely useful in obtaining powerful learning algorithms for various types of Boolean functions such as constant-depth circuits \cite{LMN93}, monotone functions \cite{BshoutyTamon:96} and intersections of halfspaces \cite{KOS:04}, and has also proved useful in developing agnostic learning algorithms for different types of Boolean functions as well \cite{KKMS05,KOS:08,BOW:08,DHK+:10}.  Several algorithmic variants of the basic low-degree algorithm \cite{Jackson:97,OdonnellServedio:07} that have been developed for other concept classes also work by constructing a low-degree PTF hypothesis.
}

Given the importance of learning algorithms that construct low-degree PTF hypotheses, it is a natural  goal to study the limitations of learning algorithms that work in this way.  On the positive side, it is well known that if there is a PTF (of constant degree $d$) that is consistent with \emph{all} the examples in a data set, then a consistent hypothesis can be found in polynomial time simply by using linear programming (with the $\Theta(n^d)$ monomials of degree at most $d$ as the variables in the LP). However, the assumption that some low-degree PTF correctly labels all examples seems quite strong; in practice data is often noisy or too complex to be consistent with a simple concept.  Thus we are led to ask:  if no low-degree PTF classifies an entire data set perfectly, to what extent can the data be learned using low-degree PTF hypoptheses?

In this paper, we address this question under the agnostic learning framework~\cite{Hau92,KSS94}. Roughly speaking, a function class $\calC$ is agnostically learnable if we can efficiently find a hypothesis that has accuracy arbitrarily close to the accuracy of the best hypothesis in $\calC$. Uniform convergence results~\cite{Hau92} imply that learnability in this model is essentially equivalent to the ability to come up with a hypothesis that correctly classifies almost as many examples as the optimal hypothesis in the function class. This problem is sometimes referred to as a ``Maximum Agreement'' problem  for $\calC$.  As we now describe, this problem has previously been well studied for the class
$\calC$ of halfspaces.

\medskip \noindent {\bf Related Work.}
The Maximum Agreement problem for halfspaces over $\R^n$ was shown to be NP-hard
to approximate within some constant factor in
\cite{amaldi-kann,bendavid}. The inapproximability factor was improved
to $84/85+\eps$ in \cite{BshBur02}, which showed that this hardness result applies even if the
examples must lie on the $n$-dimensional Boolean hypercube. Finally, a tight
inapproximability result was established independently in \cite{GR09journal}
and \cite{FGKP:journal}; these works showed that for any constant $\eps > 0$, it is NP-hard to find a halfspace consistent with $(\half+\eps)$ of the examples even if there exists a halfspace consistent with
$(1-\eps)$ of the examples.   (It is trivial to find a halfspace consistent with half of the
examples since either the constant-0 or constant-1 halfspace will suffice.)  The reduction in
\cite{FGKP:journal} produced examples with real-valued coordinates, whereas
the proof in \cite{GR09journal} yielded examples that lie on the Boolean hypercube.

Thanks to these results the Maximum Agreement problem is well-understood for halfspaces, but the situation is very different for low-degree PTFs.  Even for degree-$2$ PTFs no hardness results were  previously known,
and recent work \cite{DHK+:10} has in fact given efficient agnostic learning algorithms for low-degree PTFs under specific distributions on examples such as Gaussian distributions or the uniform distribution over $\{-1,1\}^n$  (though it should be noted that these distribution-specific agnostic learning algorithms for degree-$d$ PTFs
are not proper -- they output PTF hypotheses of degree $\gg d$).
In this paper we make the first progress on this problem, by establishing
strong hardness of approximation results for the Maximum Agreement problem for low-degree PTFs.  Our results directly imply corresponding hardness results for  agnostically learning low degree PTFs under arbitrary distributions; we present all these results below.

\ignore{We show that there is no better-than-trivial approximation algorithms for the problem which immediately implies that there is no-better-than trivial proper agnostic learning algorithm for the concept class of low degree PTFs.
In addition, for the special case of $d=2$, we strengthen our result by showing that finding the Maximum Agreement degree $2$ PTF is hard to approximate even if  the optimal hypothesis is a degree $1$ PTF (halfspace). }

\medskip
\noindent {\bf Main Results.}
\ignore{
Let us formally define PTFs before stating our main results:

\begin{definition}\label{def:ptf}
A degree-$d$ polynomial $p(x):\R^n \to \R$ is a function of the form
\[p(x)=\sum_{\text{multiset } S\subseteq[n], |S|\leq d} c_{S} \prod_{i\in S} x_i\] where each $c_S$ is a real coefficient.
A degree $d$ PTF is a Boolean-valued function $f: \R^n \to \{-1,1\}$, $f(x)=\sign(p(x))$ for
some degree-$d$ polynomia $p.$
\end{definition}

}
Our main results are the following two theorems. The first result establishes UGC-hardness of finding a nontrivial degree-$d$ PTF hypothesis even if some degree-$d$ PTF has almost perfect accuracy:

\begin{theorem}\label{thm:degreed}
Fix $\eps > 0$, $d \geq 1.$  Assuming the Unique Games Conjecture, no polynomial-time algorithm can find a degree-$d$ PTF that is consistent with $(\half + \eps)$ fraction of a given set of labeled examples in $\R^n \times \{-1,1\}$, even if there exists a degree-$d$ PTF that is consistent with a $1-\eps$ fraction of the examples.\end{theorem}

The second result shows that it is NP-hard to find a degree-2 PTF hypothesis that has
nontrivial accuracy even if some halfspace has almost perfect accuracy:

\begin{theorem}\label{thm:degree2}
Fix $\eps > 0.$  It is $\NP$-hard to find a degree-2 PTF that is consistent with
$(\half + \eps)$ fraction of a given set of labeled examples in $\R^n \times \{-1,1\}$, even if
there exists a halfspace (degree-1 PTF) that is consistent with a $1 - \eps$ fraction of the
examples.
\end{theorem}

As noted above, both problems become easy (using linear programming) if the best hypothesis is assumed to have perfect agreement with the data set rather than agreement $1-\eps$, and it is trivial to find a (constant-valued) hypothesis with agreement rate $1/2$ for any data set.  Thus the parameters in both hardness results are essentially the best possible.

These results can be rephrased as hardness of agnostic learning results in the following way:  (i) Assuming the Unique Games Conjecture, even if there exists a degree-$d$ PTF that is consistent with $1-\eps$ fraction of the examples, there is no efficient \emph{proper} agnostic learning algorithm that can output a degree-$d$ PTF correctly labeling more than $\half + \eps$ fraction of the examples; (ii) Assuming $\Pp\ne \NP$, even if there exists a halfspace that is consistent with $1-\eps$  fraction of the examples, there is no efficient agnostic learning algorithm that can find a degree-$2$ PTF correctly labeling more than $\half+\eps$ fraction of the examples.

\medskip

\noindent {\bf Organization.}  In Section~\ref{sec:pre} we present the complexity-theoretic basis (the Unique Games conjecture and the NP-hardness of Label Cover) of our hardness results.  In Section~\ref{sec:overview} we sketch a new proof of the hardness of the Maximum Agreement problem for halfspaces, and give an overview of how the proofs of Theorems~\ref{thm:degreed} and~\ref{thm:degree2} build on this basic argument.
In Sections~\ref{sec:degreed} and~\ref{sec:degree2} we prove Theorems~\ref{thm:degreed} and~\ref{thm:degree2}.

\medskip

\noindent \textbf{Notational Preliminaries:} For $n \in \mathbb{Z}_+$ we denote by $[n]$ the set $\{1,\ldots,n\}$. For $i,j \in
\mathbb{Z}_+$, $i\leq j$, we denote by $[i,j]$ the set $\{i, i+1, \ldots, j\}$.  We write $\{j:m\}$ to denote the multi-set that contains $m$ copies of the element $j$.  We write $\chi_S(x)$ to denote $\prod_{i \in S} x_i$, the monomial corresponding to the multiset $S$.

\section{Complexity-theoretic preliminaries}\label{sec:pre}

\ignore{
A halfspace is a function defined on  $\R^n\to \bits$ that can be written into the  form of  $$\sgn(\sum_{i=1}^{n} w_i x_i + \theta).$$
Efficiently learning a halfspace function is a central problem in machine learning and statistics. It is known to be PAC learnable when there is no noise~\cite{valiant,BEHW-occam,littlestone,HSW,Rivest:87,Kearns:98} or when the noise is uniformly random~\cite{BFKV}.

However, under the framework of agnostic learning where the noise can be arbitrarily bad, no (general) efficient algorithm is known yet. On the hardness
side of this problem, in \cite{GR09journal,FGKP:journal} it is shown that (even) weakly \emph{proper}  learning of halfspaces is $\NP$-hard; on the algorithm side, in~\cite{KKMS} the authors gave an efficient agnostic learning algorithm with  the assumption that the marginal distribution on the examples is uniform (or satisfies some other constraints). Their output hypothesis is not a halfspace but a higher degree polynomial threshold function (PTF). A natural question to ask is whether any algorithm that outputs high degree PTF can efficiently learn halfspaces under general distribution.  Actually  a widely used learning model, Support Vector Machine with polynomial kernel,  outputs a high degree PTF.

In additional to the this,  for high degree polynomial threshold functions itself, no hardness/algorithm result regarding the learnability  is known yet to the best of our knowledge.

In this paper we address the special case of  this problem for degree $2$ polynomial threshold functions. We prove that agnostic learning of halfspaces by degree $2$ polynomial threshold functions is hard. The following is our
result, which follows from Theorem~\ref{thm:hd},~\ref{thm:cp} and~\ref{thm:sd}.
\ignore}

\ignore{
In the framework
of agnostic learning, learnability  is essentially equivalent to  finding the hypothesis from the concept class
that correctly labels the maximum number of pairs.
A equivalent statement of our result is: agnostically weak learning of halfspaces by degree 2 polynomial
threshold function is NP-hard unless $\NP = \RP$.

As an immediate corollary, our results implies that agnostically learning degree 2 polynomial threshold function is $\NP$-hard assuming the same complexity conjecture.
}

\ignore{now explain the complexity conjecture we are using to derive our hardness results.
Our first hardness result (Theorem~\ref{thm:degreed}) is based on a reduction from the Unique Games.}

We recall the Unique Games problem that was introduced by Khot \cite{Kho02}:

\begin{definition}
A Unique Games instance $\calL$ is defined by a tuple $(U,V,E,k,\Pi)$. Here $U$ and $V$ are the two vertex sets of a regular bipartite graph  and $E$ is the set of edges between $U$ and $V$.  $\Pi$ is a collection of bijections, one for each edge: $\Pi = \{\pi_{e}:[k] \to [k]\}_{e\in E}$ where
 each $\pi_e$ is a bijection on $[k].$  A \emph{labeling} $\ell$ is a function that maps
 $U \to [k]$ and $V \to [k]$.  We say that an edge $e = (u,v)$ is \emph{satisfied} by  labeling $\ell $ if $\pi_e(\ell(v)) = \ell(u)$.  We define the \emph{value} of the Unique Games
 instance $\calL$, denoted $\opt(\calL),$  to be the maximum fraction of edges that can be satisfied by any labeling.
\end{definition}

The Unique Games Conjecture (UGC) was proposed by Khot in \cite{Kho02} and has led to many improved hardness of approximation results over those which can be achieved assuming only $\Pp \neq \NP$:

\begin{conjecture} [{\bf Unique Games Conjecture}]~\footnote{We use the statement from~\cite{KKMO07} which is equivalent to the original Unique Games Conjecture.} \label{conj:ugc} Fix any constant $\eta>0$.   For sufficiently large $k = k(\eta)$, given a Unique Games instance $\calL=(U,V,E,k,\Pi)$ that is guaranteed to satisfy one of
the following two conditions, it is $\NP$-hard to determine which condition is satisfied:
$\opt(\calL) \geq 1 - \eta,$ or $\opt(\calL) \leq \frac{1}{k^{\eta}}.$
%
\end{conjecture}

Our first hardness result, Theorem~\ref{thm:degreed}, is proved under the the Unique Games Conjecture.
Our second hardness result, Theorem~\ref{thm:degree2}, uses only the assumption that $\Pp \neq\NP$; the proof employs a reduction from the Label Cover problem, defined below.
\begin{definition}
A Label Cover instance $\calL$ is defined by a tuple $(U,V,E,k,m,\Pi)$. Here $U$ and $V$ are the two vertex sets of a regular bipartite graph  and $E$ is the set of edges between $U$ and $V$.  $\Pi$ is a collection of ``projections'', one for each edge: $\Pi = \{\pi_{e}:[m] \to [k]\}_{e\in E}$ and  $m,k$ are positive integers.  A \emph{labeling} $\ell$ is a function that maps  $U \to [k]$ and
$V \to [m]$.  We say that an edge $e = (u,v)$ is \emph{satisfied} by  labeling $\ell $ if $\pi_e(\ell(v)) = \ell(u)$.  We define the \emph{value} of the Label Cover instance, denoted $\opt(\calL)$, to be the maximum fraction of edges that can be satisfied by any labeling.
\end{definition}

We use the following theorem \cite{Raz98} which establishes NP-hardness of a ``gap'' version of Label Cover:
\begin{theorem}\label{thm:labelcover}
Fix any constant $\eta > 0.$  Given a Label Cover instance
$\calL=(U,V,E,k,m,\Pi)$ that is guaranteed to satisfy one of
the following two conditions, it is $\NP$-hard to determine which condition is satisfied:
 $\opt(\calL) =1$, or $\opt(\calL) \leq 1/m^{\eta}$.
\end{theorem}

\section{Overview of our arguments} \label{sec:overview}

To illustrate the structure of our arguments, let us begin by sketching a proof of the following hardness result for the Maximum Agreement problem for halfspaces:

\begin{proposition}\label{prop:degree1}Assuming the Unique Games Conjecture,
 no polynomial-time algorithm can find a halfspace (degree-1 PTF) that is consistent with $(\half + \eps)$ fraction of a given set of labeled examples in $\R^n \times \{-1,1\}$, even if there exists a halfspace that is consistent with a $1-\eps$ fraction of the examples.
\end{proposition}
As mentioned above, the same  hardness result (based only on the assumption that $\Pp\ne \NP$)  has already been established in~\cite{FGKP:journal,GR09journal}; indeed, we do not claim Proposition~\ref{prop:degree1} as a new result.  However, the argument sketched below is different from (and, we believe, simpler than) the other proofs; it helps to illustrate how we eventually achieve the more general hardness results Theorems~\ref{thm:degreed} and~\ref{thm:degree2}.

\medskip

\noindent
{\bf Proof Sketch for Proposition~\ref{prop:degree1}:}
We describe a reduction that maps any instance $\calL$ of Unique Games to
a set of labeled examples with the following guarantee:  if $\Opt(\calL)$ is
very close to 1 then there is a halfspace that agrees with $1-\eps$
fraction of the examples,  while if $\Opt(\calL)$ is very close to 0 then
no halfspace agrees with more than
$\frac{1}{2}+\eps$ fraction of the examples.  A reduction of this sort directly
yields Proposition \ref{prop:degree1}.

Let $\calL=(U,V,E,k,\Pi)$ be a Unique Games instance.  Each example
generated by the reduction has $(|V|+|U|)  k$ coordinates, i.e. the examples lie in
$\R^{(|U|+|V|)  k}$.  The coordinates should be viewed as being grouped together
in the following way:  there is a block of $k$ coordinates for each vertex $w$ in
$U\cup V$.  We index the coordinates of $x \in \R^{(|U|+|V|)  k}$ as
$x = (x_{w}^{(i)})$ where ${w \in U\cup V}$ and $i\in [k]$.

Given any function $f:\R^{(|U|+|V|)  k}\to \{-1,1\}$ and vertex $w \in U \cup V$, we write $f_w$ to denote the restriction of $f$ to the $k$ coordinates $(x^{(i)}_w)_{i \in [k]}$ that is obtained by setting all other coordinates $(x^{(j)}_{w'})_{w' \neq w}$ to 0.  Similarly, for $e = \{u,v\}$ an edge in $U \times V$, we write $f_e$ for the restriction that fixes all coordinates $(x_{w'}^{(i)})_{w' \notin e}$ to $0$ and leaves the $2k$ coordinates $x^{(i)}_u, x^{(i)}_v$ unrestricted.

For every labeling $\ell : U\cup V \to [k]$ of the instance, there is a
corresponding halfspace  over $\R^{(|V|+|U|)  k}$
$$  \sgn(\sum_{u\in U} x_{u}^{(\ell(u))} - \sum_{v\in V} x_{v}^{(\ell(v))}).$$

Given a Unique Games instance $\calL$, the reduction constructs a distribution $\calD$ over labeled examples such that
if $\Opt(\calL)$ is almost 1 then the above halfspace has very high accuracy w.r.t. $\calD$, and any halfspace that has accuracy at least $\half + \eps$ yields a labeling that satisfies a constant fraction of edges in
$\calL.$
A draw from $\calD$ is obtained by first selecting a uniform random edge $e =\{u,v\}$ from $E$, and then making a draw from $\calD_e$, where
${\calD_e}$ is a distribution over labeled examples that we describe below.

Fix an edge $e =(u,v)$.  For the sake of exposition, let
us assume the mapping $\pi^{e} \in \Pi$ associated with $e$ is the identity permutation, i.e. $\pi^e(i)=i$
for every $i \in [k].$  The distribution $\calD_e$ will have the following properties:

\begin{itemize}
        \item [(i)] For every $(y,b)$ in the support of $\calD_e$, all coordinates $y_w^{(i)}$ for every vertex $w \notin e$ are zero.

        \item [(ii)] For every label $i \in [k]$, the halfspace $\sign(x_{u}^{(i)} -x_{v}^{(i)})$ has accuracy
        $1-\eps$ w.r.t.$\calD_e.$
        \item [(iii)] If $\sign(f_e)$ is a halfspace that has accuracy at least
                $\frac{1}{2}+\eps$ w.r.t. $\cD_e$, then the functions $f_u,f_v$ can each be individually ``decoded'' to a ``small'' (constant-sized) set $S_u, S_v \subseteq [k]$ of labels such that $S_u \cap S_v \neq \emptyset$ (so a labeling that satisfies a nonnegligible fraction of edges in expectation can be obtained simply by choosing a random label from $S_w$ for each $w$ -- such a random choice will satisfy each edge's bijection with constant probability, so in expectation will satisfy a constant fraction of constraints).
\end{itemize}

Let us explain item (iii) in more detail.  Since the distribution $\cD_e$ is supported on vectors $y$ that have the
$(y^{(i)}_w)_{w \notin e}$ coordinates all 0, the distribution $\cD_e$ only ``looks at'' the restriction $f_e$ of $f$, which is a halfspace on $\R^{2k}$.  Thus achieving (iii) can be viewed as solving a kind of property testing problem which may loosely be described as ``Matching dictator testing for halfspaces.''  To be more precise, what is required is a distribution $\cD_e$ over $2k$-dimensional labeled examples  and a ``decoding'' algorithm $A$ which takes as input a $k$-variable halfspace and outputs a set of coordinates.  Together these must have the following properties:

\begin{itemize}
        \item (Completeness) If $f_e(x) = x_{u}^{(i)}-x_{v}^{(i)}$ then $\sign(f_e(y))=b$ with probability $1-\eps$ for $(y,b) \sim\cD_e$;
        \item (Soundness)  If $f_e$ is such that $\sign(f_e(y))=b$ with probability
        at least $1/2 + \eps$ for $(y,b)$ drawn from $\cD_e$,
        then the output sets $A(f_u)$, $A(f_v)$ of the decoding algorithm (when it is run on $f_u$ and $f_v$ respectively) are two small sets that intersect each other.
\end{itemize}
Testing problems of this general form are often referred to as
{\em Dictatorship Testing}; the design and analysis of such tests is a recurring theme in hardness of approximation.

We give a ``matching dictator test for halfspaces'' below.  More precisely, in the following figure we describe the distribution $\cD_e$ over examples (the decoding algorithm $A$ is
 described later).

\begin{center}
\fbox{
\parbox{3 in}{
\begin{center}
\textbf{$\calT_1$:  Matching Dictatorship Test for Halfspaces}
\end{center}

\noindent {\bf Input:}  A halfspace $f_e: \R^{2k} \to \R$.

\medskip

Set $\eps :=\frac{1}{\log k}$, $\delta:= 1/2^k$.

\begin{enumerate}
 \item  Generate independent 0/1 bits $a_1,a_2,\ldots,a_k$ each with $\E[a_i]=\eps.$
Generate $2k$ independent $N(0,1)$ Gaussian random variables: $h_1,h_2\ldots,h_k,g_1, g_2\ldots,g_k$.
Generate a random bit $b\in \{-1,1\}.$
\item Set $r = (a_1 h_1+g_1,\ldots, a_k h_k +g_k, g_1,\ldots,g_k)$  and $\omega = (1,\ldots,1,0,\ldots,0) \in \R^{2k}$ to be the vector whose first $k$ coordinates are 1 and last $k$ coordinates are 0.
\item Set $y =r+b\delta\omega.$ The result of a draw from $\cD_e$ is the
labeled example $(y,b)$.
\end{enumerate}

The test checks whether $\sgn(f_e(y))$ equals $b.$

}
}
\end{center}

It is useful to view the test in the following light:
Let us write
$f_e (x)$ as $\theta + \sum_{i=1}^k w_u^{(i)} x_u^{(i)} + \sum_{i=1}^k w_{v}^{(i)} x_{v}^{(i)}$,
and let us suppose that $ \sum_{i=1}^k |w_u^{(i)}| =1$ (as long as some $w_u^{(i)}$ is nonzero this is easily achieved by rescaling;  for this intuitive sketch we ignore the case that all $w_u^{(i)}$ are 0, which is not difficult to handle).  Then we have
$f_e(y) = f_{e}(r) +b \delta,$  and we may view the test as randomly choosing one of the two inequalities $f_{e}(r) -\delta < 0$, $f_{e}(r) -\delta > 0$ and checking that it holds.
Since at least one of these inequalities must hold for every $f_e$, the probability that $f_e$ passes the test is $\half + \half \Pr_r[f_{e}(r) \in [-\delta,\delta)]$.  This interpretation will be useful both for analyzing completeness and soundness of the test.

For completeness,  it is easy to see that the ``matching dictator'' function $f_e(x) = x_u^{(i)} - x_v^{(i)}$ has $f_e(r) = a_i h_i$ and thus $\Pr[f_e(r)=0] = 1-\eps,$ so this function indeed  passes the test with probability $1-\eps$.

The soundness analysis, which we now sketch, is more involved.
Let $f$ be such that $\Pr_r[f_{e}(r) \in [-\delta,\delta)] \geq 2 \eps.$  Since $f_{e}(r)=
\sum_i  (w_u^{(i)} + w_v^{(i)}) g_i + \sum w_u^{(i)} a_i h_i$ and $g_i,h_i$ are i.i.d. Gaussians, conditioned on a given outcome of the $a_i$-bits the value $f_e(r)$ follows the Gaussian distribution with mean 0 and variance $\sum (w_u^{(i)} + w_v^{(i)})^2 + \sum (a_i w_u^{(i)})^2$.  Now recall that an $N(0,\sigma)$
Gaussian random variable lands in the interval $[-t,t]$ with probability at most
 $O(t/\sigma)$.  So any $a$-vector for which the variance $\sum (w_u^{(i)} + w_v^{(i)})^2 + \sum (a_i w_u^{(i)})^2$ is not ``tiny'' can contribute only a negligible amount to the overall probability
 that $f_e(r)$ lies in $[-\delta,\delta)$ (recall that $\delta$ is extremely tiny).  Since by assumption $\Pr_r[f_e(r) \in [-\delta,\delta)]$ is non-negligible (at least $2 \eps$), there must be a non-negligible fraction of $a$-vector outcomes that
 make the variance $\sum (w_u^{(i)} + w_v^{(i)})^2 + \sum (a_i w_u^{(i)})^2$ be ``tiny.''
 This implies that there must be only a ``few'' coordinates $w^{(j)}_u$ for which $|w^{(j)}_u|$ is not tiny (for if there were many non-tiny $w^{(j)}_u$ coordinates, then $\sum_i (w_u^{(i)} a_i)^2$ would be non-tiny with probability nearly 1 over the choice of the $a$-vector).  Moreover, $w^{(i)}_u + w_{(i)}^v$
 must be $\approx 0$ for each $i$, so for each $i$ the magnitudes $|w^{(i)}_u|$ and $|w^{(i)}_v|$ must be nearly equal; and in particular, each $|w^{(i)}_u|$ is large if and only if
 $|w^{(i)}_v|$ is large.
 Finally, since $\sum_i |w^{(i)}_u|$ equals 1  some $w^{(i)}_u$'s must be large (at least $1/k$).

 With these facts in place, the appropriate decoding algorithm $A$ is rather obvious:  given $f_u =
 \theta + \sum_{i=1}^k w^{(i)}_u x^{(i)}_u$ as input, $A$ outputs the set $S_u$ of those coordinates $i$
 for which $|w^{(i)}_u|$ is large (and similarly for $f_v$).  This set cannot be too large since
   $\sum_{i=1}^k |w^{(i)}_u|$ equals 1.  Now a labeling that satisfies edge $e$
 with non-negligible probability can be obtained by outputing a random element from $S_u$
 and a random element from $S_v$; since these sets are small there is a non-negligible probability that the labels will match as required.
 This concludes the proof sketch of Proposition~\ref{prop:degree1}.\qed

\medskip

\noindent {\bf Overview of the proofs of Theorems~\ref{thm:degreed} and~\ref{thm:degree2}.}
For Theorem~\ref{thm:degreed} (hardness of properly learning degree-$d$ PTFs), we must deal with the additional complication of handling the cross-terms such as $x_u^{(i)} x_v^{(j)}$ between $u$-variables
and $v$-variables that may be present in degree-$d$ PTFs. As an example of how such cross-terms can cause problems, observe that the degree-$3$ polynomial $f_e = (x_u^{(i)} - x_v^{(i)}) \sum (x_u^{(i)})^2$ would pass the test $\calT_1$ with high probability, but this polynomial has $f_v = 0$ so there is no way to successfully ``decode'' a good label for $v$.  To get around this, we
modify the test $\calT_1$ to set $y =(a_1 h_1 + g_1^d + b\delta,a_2 h_2 + g_2^d + b\delta, \ldots, a_k h_k + g_k^d + b \delta, g_1,\ldots,g_k)$; intuitively this modified test checks whether the
polynomial $f_e$ is of the form $x_u^{(i)} - (x_v^{(i)})^d$.  The bulk of our work is in analyzing the soundness of this test; we show that any polynomial $f_e$ that passes the modified test with probability significantly better than $1/2$ must have almost no coefficient weight on cross-terms, and that in fact the restricted polynomials $f_u,f_v$ can each be decoded to a small set in such a way that there is a matching pair as desired.\ignore{

 \rnote{explain why
  this is the right/useful thing to do} to do this we modify the test $\calT_1$ to set $y =(a_1 h_1 + g_1^d + b\delta,a_2 h_2 + g_2^d + b\delta, \ldots, a_k h_k + g_k^d + b \delta, g_1,\ldots,g_k)$.}\ignore{and check (as before) that $\sgn(f_{e}(y))=b.$}\ignore{  Our analysis shows (among other things)
that any polynomial that passes this test with probability significantly better than $1/2$ must
have almost no coefficient weight on cross-terms.}  We give a complete description and analysis of our Dictator Test and prove Theorem~\ref{thm:degreed} in~Section~\ref{sec:degreed}.

For Theorem~\ref{thm:degree2}, a first observation is that the test $\calT_1$ in fact already has soundness $3/4+\eps$ for degree-$2$ PTFs. To see this, we begin by
writing the degree-2 polynomial $f_{e}(x)$ as $\theta + f_1(x) + f_2(x)$ where $f_1(x)$ is the linear (degree 1) part  and $f_2(x)$ is the quadratic (degree 2) part (note that $f_1$ is an odd function and $f_2$ is an even function).
We next observe that since any vector $r$ is generated with the same probability as $-r$, the test may be viewed as randomly selecting one of the following 4 inequalities to verify:
$f_{e}(r+ \delta \omega) > 0,$
$f_{e}(r - \delta \omega) < 0,$
$f_{e}(-r+\delta \omega) > 0,$
$f_{e}(-r-\delta \omega) <0.$
If all four inequalities hold, then combining $f_{e}(r+ \delta \omega) > 0$ with $f_{e}(-r -\delta \omega) <0$ we get that $f_1(r+\delta \omega) >0$ and combining $f_{e}(r-\delta \omega) < 0$ with $f_{e}(-r+\delta \omega)>0$ we get $f_1(r-\delta \omega)<0$.  Consequently, if a degree-$2$ polynomial $f_e$ passes the test with probability $3/4 + \eps$, then by an averaging argument, for at least an $\eps$ fraction of the $r$-outcomes all four of the inequalities must hold.  This  implies that for an $\eps$ fraction of the $r$'s we must have $f_1(r+\delta \omega) > 0$ and $f_1(r-\delta \omega)<0$, and so the degree-$1$ PFT $f_1$ must pass the Dictator Test $\calT_1$ with probability at least $1/2+\eps$.  This essentially reduces to the problem of testing degree-1 PTFs, whose analysis is sketched above.

To get the soundness down to $1/2$ more work has to be done.  Roughly speaking,  we modify the test by checking that $\sgn(f(k_1 r + k_2 \delta \omega)) = \sgn(k_2)$ for  $k_1,k_2$ generated from a  carefully constructed  distribution in which $k_1, k_2$ can assume many different possible orders of magnitude.  Using these many different possibilities for the magnitudes of $k_1,k_2,$ a careful analysis (based on carefully combining inequalities in a way that is similar to the previous paragraph, though significantly more complicated) shows that if a polynomial passes the test with probability $1/2 + \eps$ fraction then it can be ``decoded'' to a small set of coordinates.
In addition to this modification, to avoid using the Unique Games Conjecture we employ the ``folding trick'' that is proposed in~\cite{GKS10journal,KS08} to ensure consistency across different vertices. One benefit of using this trick is that with it, we only need to design a test on one vertex instead of an edge.\footnote{The reason that we can not use ``folding'' for our first result on low-degree PTFs, roughly speaking, is that such a folding does not seem able to handle cross-terms of degree greater than 2.}
The complete proof  of Theorem~\ref{thm:degree2} appears in~Section~\ref{sec:degree2}.

\section{Hardness of proper learning noisy degree-$d$ PTFs:  Proof of Theorem~\ref{thm:degreed}}\label{sec:degreed}

\subsection{Dictator Test}
Let $f:\R^{2n} \to \R$ be a $2n$-variable degree-$d$ polynomial over the reals. The key gadget in our UG--hardness reduction is
a \emph{dictator test} of whether $f$ is of the form $\sgn(x_i- x^d_{n+i})$ for some $i\in [n]$. More concretely, our dictator
test queries the value of $f$ on a {\em single} point $y \in \R^{2n}$ and decides to accept or reject based on the value
$\sgn(f(y))$.

\vspace{-0.5cm}

\begin{center}
\fbox{
\parbox{3 in}{
\begin{center}
\textbf{$\calT_d$: Matching Dictator Test for degree-$d$ PTFs}
\end{center}

{\bf Input:} A degree-$d$ real polynomial $f : \R^{2n} \to \R$.

\smallskip

Set $\beta := 1/\log n$ and $\delta := 2^{-n^2}$.

\begin{enumerate}
\item  Generate $n$ i.i.d. bits $a_i \in \{0,1\}$ with $\Pr[a_i = 1] = \beta$, $i \in [n]$.
Generate $2n$ i.i.d. $N(0,1)$ Gaussians $\{h_i, g_i\}_{i=1}^n$. Generate a uniform random bit $b\in \{-1,1\}.$
\item Set $y = (y_i)_{i=1}^{2n}$ where $y_i = a_ih_i+g_i^d+b\delta$ and $y_{n+i} = g_i$, $i \in [n]$.
\item Accept iff $\sgn(f(y))= b.$
\end{enumerate}
} }
\end{center}

\vspace{-0.1cm}

We can now state and prove the properties of our test. The completeness is straightforward.

\begin{lemma}[Completeness]  The polynomial $f(x) = x_i - x_{n+i}^d$ passes the test with probability at least $1-\beta$.
\end{lemma}
\begin{proof} Note that $f(y) = a_i h_i + b \delta$. Hence if $a_i=0$ we have $\sign (f(y)) = b$ and this happens with
probability $1-\beta$.
\end{proof}

To state the soundness lemma we need some more notation.
For a degree-$d$ polynomial $f(x)= \littlesum_{S \subseteq [n], |S|\leq d} c_S \cdot \chi_S(x)$ we denote $\weight(f) =
\littlesum_{S \neq \emptyset} |c_{S}|$. For $\theta>0$, we define $I_{\theta}(f):= \{ i \in [n] \mid \exists S \ni i \textrm{
s.t. } |c_{S}|\geq \theta \cdot \weight(f)/{n+d\choose d}\}.$ Note that for $\theta \in [0,1]$ we have that $I_{\theta}(f) \neq
\emptyset$, since there are ${n + d \choose d}$ nonempty monomials of degree at most $d$ over $x_1,\dots,x_n$.

Let $f:\R^{2n} \to \R$ be a $2n$-variable polynomial $f(x)= \littlesum_{S \subseteq [2n], |S|\leq d} c_S \cdot \chi_S(x)$ fed
as input to our test. We will consider the restrictions obtained from $f$ by setting the first (resp. second) half of the
variables to $0$. In particular, for $x = (x_1, \ldots, x_{2n})$ we shall denote $f_1(x_1, \ldots, x_n) = f(x_1, \ldots, x_n,
\mathbf{0}_n)$ and $f_2(x_{n+1}, \ldots, x_{2n}) = f(\mathbf{0}_n, x_{n+1}, \ldots, x_{2n})$.

We are now ready to state our soundness lemma. The proof of this lemma poses significant complications
and constitutes the bulk of the analysis in this section.

\begin{lemma}[Soundness]\label{lem:soundness} Suppose that  $f(x)= \littlesum_{S \subseteq [2n], |S|\leq d} c_S
\cdot \chi_S(x)$ passes the test with probability at least $1/2 + \beta $. Then  for $f_1, f_2$ as defined above, we have
$|I_{0.5}(f_{1})|\leq 1/\beta^2, |I_1(f_2)|\leq 1/\beta^2$. In addition, every $i\in [n]$ such that $n+i\in I_{1}(f_2)$ also
satisfies $i\in I_{0.5}(f_1)$.
\end{lemma}
\begin{proof}
We can assume that $\weight(f) > 0$, since otherwise $f$ is a constant function, hence passes the test with probability exactly
$\half.$ Since our test is invariant under scaling, we can further assume that $\weight(f)=1$.

Let $x \in \R^{2n}$. By definition, $f_1(x)= \littlesum_{S\subseteq [n]} c_{S} \cdot \chi_S(x)$ and $f_2(x) =
\littlesum_{S\subseteq [n+1,2n]} c_{S} \cdot \chi_S(x).$ We can write
$$f(x) = f_1(x) + f_2(x) + f_{12}(x)$$ where $f_{12}(x) = \littlesum_{S\subseteq [2n], S \cap[n] \neq \emptyset, S\cap
[n+1,2n]\neq \emptyset} c_{S} \cdot \chi_S(x).$

\smallskip

Let us start by giving a very brief overview of the argument. The proof proceeds by carefully analyzing the structure of the
coefficients $c_S$ for the subfunctions $f_1, f_2, f_{12}$. In particular, we show that the total weight of the cross terms
(i.e. $\weight(f_{12})$) is negligible, and that the weight of $f$ is roughly equally spread among $f_1$ and $f_2$. Moreover,
the coefficients of $f_1, f_2$ are either themselves negligible or ``matching'' (see inequalities (i)-(iv) below). Once these
facts have been established, it is not hard to complete the proof.

The main step towards achieving this goal is to relate the coefficients $c_S$ with the coefficients of an appropriately chosen
restriction of $f$, obtained by carefully choosing an appropriate value of $a\in\bn$. We start with the following crucial claim:

\begin{claim} \label{claim:small-norm} Suppose $f$ passes the test with probability at least $1/2+\beta$.
Then there exists $\alpha' \in \bn$ such that
\vspace{-0.2cm}
$$\| f_{\alpha'}\|_2 \leq 2^{-n}\cdot\log^{d^2} n.$$
\end{claim}

\begin{proof}[Proof of Claim~\ref{claim:small-norm}]
Let us start be giving an equivalent description of the test. Denote $\omega =(\mathbf{1}_n,\mathbf{0}_n)\in \R^{2n}$, $r =
(r_i)_{i=1}^{2n}$ with $r_i = a_i h_i + g_i^d$ and $r_{n+i} = g_i$, $i \in [n]$. Note that $y = r+ (b \delta)\omega$. Then the
Dictator Test $\calT_d$ is as follows:

\vspace{-0.2cm}
\begin{itemize}
\item Generate $r$, and with probability $1/2$, test whether $f(r+\delta \omega) \geq 0$; otherwise test $f(r-\delta \omega)<0$.
\end{itemize}
\vspace{-0.2cm}

\noindent Hence, since $f$ passes with probability $1/2+\beta$, with probability at least $2\beta$ over the choice of $r$, the
following inequalities are simultaneously satisfied: \vspace{-0.2cm}
$$ f(r+\delta \omega) \geq 0 ; f(r-\delta \omega) < 0.$$
We now upper bound $|f(r + \delta \omega) -f(r)|$: \vspace{-0.2cm}
\begin{multline*}\label{eqn:tob}
    \big|f(r + \delta \omega) -f(r)\big| \\=
| \littlesum_{|S|\leq d} c_S \cdot \Big(\littleprod_{i \in S\cap[n]} (r_i + \delta) \\ \cdot \littleprod_{j\in S\cap[n+1,2n]}{r_j} - \littleprod_{i\in S}{r_i}  \Big) | \\
\leq  \littlesum_{1\leq |S|\leq d} |c_{S}| \cdot \Big( \littlesum_{\emptyset \neq T \subseteq S\cap[n]}  \delta^{|T|} \cdot
\littleprod_{i \in S\setminus T} |r_i| \Big) \\ \leq \littlesum_{1\leq|S|\leq d} |c_S| \cdot 2^{|S|} \cdot  \Big(\delta \cdot
\littleprod_{i\in S:r_i \geq 1} |r_i|\Big)
\end{multline*}

\vspace{-0.2cm}

The last inequality follows from the fact that there are at most $2^{|S|}$ terms in the second summation each bounded from
above by $\delta \cdot \littleprod_{i\in S:r_i \geq 1} |r_i|.$

We now claim that with probability at least $1-n^{-1}$ over the choice of $r$ it holds $M:=\max_{i\in [2n]} |r_i| \leq \log^d
n$. To see this note that if $\max_{i \in [n]} \{|g_i|, |h_i|\} \leq c$ then $M\leq 2c^d$. Now recall that for $g\sim N(0,1)$
and $c>2$ we have $\Pr[|g| > c] \leq e^{-c^2/2}$. The claim follows by fixing $c= \Theta (\log^{1/2} n)$ and taking a union
bound over the corresponding $2n$ events.

Therefore, with probability $1-n^{-1}$ over the choice of $r$, we have \vspace{-0.2cm}
\[
  \big|f(r + \delta \omega) -f(r)\big| \leq \delta \cdot 2^d \cdot (\log n)^{d^2} \cdot \weight(f) \leq 2^{-n}.
\]

\vspace{-0.2cm} \noindent Analogously we obtain that $|f(r)-f(r-\delta \omega)| \leq 2^{-n}.$ We conclude that with probability
$2\beta-n^{-1} \geq \beta$ over $r$
\begin{equation} \label{eqn:smallb}
\big|f(r)\big| \leq 2^{-n}.
\end{equation}
Recall that $r$ is a random vector that depends on $a, g,h$. For every realization of $a \in \bn$, we denote the corresponding
restriction of $f$ as $f_a(g,h)$; note that $f_a(g,h)$ is a degree $d^2$ real polynomial over Gaussian random variables. Let us
denote $\|f_a\|_2 := \Ex_{g,h}[f_a(g,h)^2]^{1/2}$.

At this point we appeal to an analytic fact from~\cite{CW01}: low degree polynomials over independent Gaussian inputs have
good anti-concentration. In particular, an application of Theorem~\ref{thm:carberywright} for $f_a(g,h)$ yields that {for
all} $a \in \bn$ it holds
$$\Pr_{g,h}[|f_a(g,h)| \leq 2^{-n}] \leq d^2 \cdot (2^{-n}/\|f_a\|_2)^{1/d^2}.$$
Combined with (\ref{eqn:smallb}) this gives \vspace{-0.2cm}
\begin{eqnarray*}
\beta &\leq& \Pr_{a,g,h}\left[|f_a(g,h)| \leq 1/2^n\right]\\ 
&\leq& \Ex_{a}\left[d^2 \cdot (2^{-n}/\|f_a\|_2)^{1/d^2}\right].
\end{eqnarray*}

\noindent Now let us fix $a':= \arg \min _{a\in \bn} \|f_a\|_2$; the above relation implies
$\big(2^{-n}/\|f_{a'}\|_2\big)^{1/d^2} \geq \beta$ or $\|f_{a'}\|_2 \leq 2^{-n}(1/\beta)^{d^2}$ as desired. This completes the
proof of Claim~\ref{claim:small-norm}.
\end{proof}

\noindent Since $a'$ is fixed, we can express $f_{a'}$ as a degree-$d^2$ polynomial over the $g_i$'s and $h_i$'s. Let us write
$$f_{a'} = \littlesum_{T, T'} {w_{T,T'} \cdot \littleprod_{i\in T} {g_i} \cdot \littleprod_{i\in T'}{h_i}}$$
where $T, T' \subseteq [n]$ are multi-sets satisfying $|T|+|T'| \leq d^2$ and $w_{T,T'} = w_{T,T'}(a')$. Since $f_{a'}$ has
small variance, intuitively each of its coefficients should also be small. The following simple fact establishes such a
relationship:

\begin{fact}\label{lem:btvar}
Let $f:\R^l \to R$ be a degree-$d$ polynomial $f(x) = \sum_{|S|\leq d} c_{S} \cdot \chi_S(x)$ and $\mathcal{G} \sim N(0,1)^l$.
For all $T\subseteq [l]$ we have $\|f(\mathcal{G})\|_2 \geq d^{-d} \cdot |c_T|  / {l+d\choose d}$.
\end{fact}

\begin{proof}[Proof of Fact~\ref{lem:btvar}]

The fact follows by expressing $f$ in an appropriate orthonormal basis. Let $\{H_S\}_{S\subseteq [l], |S| \leq d}$ be the set of Hermite polynomials of degree at most $d$ over $l$ variables, let and $f(x) = \sum_{|S|\leq d} \hat{f}(S) H_S(x)$ be the
Hermite expansion of $f$. Then, $\|f(\mathcal{G})\|_2^2 = \littlesum \hat{f}(S)^2$ which clearly implies that
$\|f(\mathcal{G})\|_2 \geq \max_S |\hat{f}(S)|$.

Fix an $S \subseteq [l]$ with $|S|\leq d$. By basic properties of the Hermite polynomials (see e.g. \cite{Janson:97}) we have that $H_S(x) = \sum_{U
\subseteq S} h_S^U \cdot \chi_{U}(x)$ with $|h_S^U| \leq d^d$. Hence, for a fixed $T\subseteq[l]$, $c_{T}$ can be written as
$\sum_{S \supseteq T} h_{S}^T \hat{f}(S)$. Since $S \subseteq[l]$ and $|S| \leq d$, there are at most ${l+d\choose d}$ terms in
the summation. Therefore, it must be the case that there exists some $S$ such that $|\hat{f}(S)| \geq d^{-d} \cdot |c_T| /
{l+d\choose d}$. This completes the proof.
\end{proof}

\noindent {\em Notation:} For the remaining of this proof we will be interested in the coefficients $w_{T,T'}$ for $T' =
\emptyset$. For notational convenience we shall denote $w_{T} := w_{T, \emptyset}$.

We now claim that {for all $T$} we have
\begin{equation} \label{eqn:smallw}
|w_{T}|\leq  n^{-10d}.
\end{equation}
Using Fact~\ref{lem:btvar}, if this were not the case we would get a contradiction with Claim~\ref{claim:small-norm}.

At this point we establish the relationship between the $w_T$'s and the coefficients $c_S$ of $f$ in our original basis
$\{\chi_S\}$.

By definition, the restriction obtained from $f_{a'}(g,h)$ by setting the $h_i$ variables to $0$ is identical to the
function $f(g_1^d,\ldots,g_n^d, g_1,\ldots,g_n)$. Therefore we have

\begin{eqnarray} 
&& \littlesum_{T\subseteq [n]} {w_{T}\cdot \littleprod_{i\in T} g_i} = \label{eq:w-vs-c} \\
&&
\littlesum_{S\subseteq [2n]} {c_{S} \cdot {\littleprod_{i
\in S\cap[n]} g_i^d} \cdot {\littleprod_{(n+i)\in S} g_i}} \nonumber
\end{eqnarray}

For any fixed $T$ in the LHS of $(\ref{eq:w-vs-c})$ there is an equivalence class of sets $S$ in the RHS such that the monomial
${\littleprod_{i \in S\cap[n]} g_i^d} \cdot {\littleprod_{(n+i)\in S} g_i}$ equals $\littleprod_{i\in T} g_i$. It is clear that
$w_T$ equals $\littlesum_S c_S$, where the sum is over all $S$ in the equivalence class. In fact, the structure of the equivalence classes is
quite simple, as established by the following claim:

\begin{claim}\label{lem:expand}
For any $S_0 \neq S_1\subseteq [2n]$ of size at most $d$, if
\begin{multline}\label{eqn:equal}
\littleprod_{i \in S_0 \cap [n]} {g_i^d} \cdot \littleprod_{n+j\in S_0,j\in[n]}{g_j}\\ = \littleprod_{i\in S_1\cap[n]} g_i^d \cdot
\littleprod_{n+j\in S_1, j \in [n]} g_j,
\end{multline}
then  there exists some  $\ell \in [n]$ such that $S_0=\{\ell\}$ and $S_1=\{n+\ell:d\}$ or vice versa.
\end{claim}

\noindent {\emph Proof of Claim~\ref{lem:expand}.}
Consider the following two complementary cases.
\begin{itemize}
\item $S_0\cap[n]\neq S_1\cap [n]$. Without loss of generality, we can assume that there is some $\ell \in S_0\cap[n]$ with
$\ell \notin S_1$. (Otherwise the role of $S_0, S_1$ can be reversed.) Then to make \eqref{eqn:equal} hold, it must be the case
that $S_1$ contains $d$ copies of $n+\ell$. Now, since $|S_1|\leq d$, it can only be the case that $S_1=\{n+\ell:d\}$, which
implies that $S_0  = \{\ell\}$.

\item  $S_0\cap [n+1,2n] \neq S_1 \cap [n+1,2n]$. We may assume that there is some $\ell \in [n]$ such that $(n+\ell)\in S_0$.
Then, for \eqref{eqn:equal} to hold, it must be the case that $\ell \in S_1$. Hence, it must be the case that $S_1 =
\{n+\ell:d\}$ (since $g_{\ell}$ is raised to the $d$th power in the RHS of \eqref{eqn:equal}); this in turns enforces $S_0 =
\{\ell\}$. \qed
\end{itemize}

Claim~\ref{lem:expand} implies the following relation between the coefficients $c_S$ and $w_{T}$:
\begin{itemize}
\item[(A)] If $T = \{i:d\}$, for some $i\in [n]$, then we have
$w_T = c_{S_1} +  c_{S_2}$ with $S_1 = \{i\}$ and $S_2 = \{n+i:d\}$.

\item[(B)] If $T$ is not of the above form, then there exists a multi-set $S\subseteq[2n]$, $|S|\leq d$,
where $S \neq \{i\}$ and $S \neq \{n+i:d\}$ for any $i\in [n]$, such that $T$ equals $\{i:d \mid i \in S \} \cup \{i \mid
n+i\in S\}$. In this case, we have $w_{T} = c_{S}.$
\end{itemize}

We are now ready to establish the desired bounds on the coefficients of the subfunctions $f_1, f_2, f_{12}$.
\begin{itemize}
\item[(i)] For all $S\subseteq[n]$ with $|S|\geq 2$, (\ref{eqn:smallw}) and (B) yield
$|c_S| \leq n^{-10d}.$

\item[(ii)]  For all $S\subseteq[n+1,2n]$ with $S\neq \{n+i:d\}$ for some $i\in[n]$, (\ref{eqn:smallw}) and (B) yield
$|c_S| \leq n^{-10d}.$

\item[(iii)] For all $i\in [n]$,  by (\ref{eqn:smallw}) and (A) we obtain
$\big||c_{\{i\}}| -|c_{\{n+i:d\}}| \big|\leq       |c_{\{i\}} + c_{\{n+i:d\}}|\leq n^{-10d}.$

\item[(iv)] For all $S$ such that $S\cap[n]\neq \emptyset$ and $S\cap[n+1,2n] \neq \emptyset$,
(\ref{eqn:smallw}) and (B) yield
$|c_{S}| \leq n^{-10d}.$
\end{itemize}
\noindent Since the coefficients of $f_1, f_2$ are either very small
(cases (i), (ii) above) or matching (case (iii)), we get
$\big| \weight(f_1) - \weight(f_2) \big| \leq n^{-10d} \cdot {n+d\choose d}  \leq n^{-1}.$
Moreover, since every coefficient of $f_{12}$ is small (case (iv)), we deduce that
$\weight(f_{12})\leq n^{-10d} \cdot {2n+d\choose d} \leq n^{-1}.$
Recalling that $\weight(f_1)+\weight(f_2)+\weight(f_{12})=\weight(f) =1$, we get
$\weight(f_1) + \weight(f_2) \geq 1-\frac{1}{n}.$
Combining these bounds, we get that 
\begin{equation} \label{eqn:weight-final}
0.51\geq \weight(f_1), \weight(f_2)\geq 0.49.
\end{equation}

Now fix an $i\in [n]$ with $(n+i) \in I_1(f_2)$. The above inequality implies that there must exist some $S\ni(n+i)$ such that
$|c_S| \geq 0.49/{n+d\choose d}$. By (ii)\ignore{\eqref{eqn:bigdd}}, we deduce that it can only be the case that $S$ equals $\{n+i:d\}$ (as
all other coefficients in $f_2$ are very small). Moreover, (iii)\ignore{\eqref{eqn:match}} implies that $|c_i| \geq 0.48{n+d\choose d}^{-1}$,
hence $i \in I_{0.5}(f_1)$ (recalling that $\weight(f_1)\leq 0.51$). So we have $|I_1(f_2)|\leq |I_{0.5}(f_1)|$ and it remains to
bound from above the size of $I_{0.5}(f_1)$ by $\beta^{-2}$.

Suppose (for the sake of contradiction) that $|I_{0.5}(f_1)| \geq \beta^{-2}$. Since $\weight(f_1)\geq 0.49$, every $j\in
I_{0.5}(f_1)$ comes from the set $S=\{j\}$ (as all the other coefficients of $f_1$ are too small). Consider all possible
realizations of $a\in \bn$. With probability $1- (1-\beta)^{|I_{0.5}(f_1)|} \geq 1- n^{-1}$ over the choice of $a$, there
exists $i\in I_{0.5}(f_1) $ with $a_i = 1.$ Fix such an $i$. By the definition of $I_{0.5}(f_1)$, we must have $|c_{\{i\}}| \geq 0.5 \cdot 0.49{n+d\choose d}^{-1} \geq 0.2 \cdot {n+d\choose d}^{-1}.$ Hence, there will be a degree-$1$ monomial in the
expansion of $f_{a}$  as a polynomial over $g$ and $h$ whose coefficient has absolute value at least $0.2 \cdot {n+d\choose
d}^{-1}$.

The aforementioned and Fact~\ref{lem:btvar} imply that with probability $1-n^{-1}$ over $a$ it holds
$$\|f_a\|_2\geq \frac{0.2}{{n+d\choose d}}\cdot \frac{1}{{2n+d^2\choose d^2}(d^2)^{d^2}}\geq \Omega(\frac{1}{n^{2d^2}}).$$

By Theorem~\ref{thm:carberywright} and the fact that $\weight(f)=1$ we get that $\Pr_{a,g,h}\left[ |f_a(g,h)| \leq 2^{-n} \right]$ is at most $n^{-1} + O(d^2 \cdot n^2 \cdot 2^{-n/d^2}) = o(\beta)$,
which contradicts (\ref{eqn:smallb}). This completes the proof of Lemma~\ref{lem:soundness}.
\end{proof}

\subsection{Hardness reduction from  Unique Games} \label{sec:hrug}

With the completeness and soundness lemmas in place, we are ready to prove Theorem~\ref{thm:degreed}. The hardness reduction is
from a Unique Games Instance $\calL(U,V,E,\Pi,k)$ to a distribution of positive and negative examples. The examples lie in
$\R^{(|U|+ |V|)k}$ and are labeled with either ($+1$) or ($-1$). Denote $\dim = (|U| + |V|) k$.

For $w\in U\cup V$ and $x\in \R^{\dim}$,  we use $x_{w}^{(i)}$ to denote the coordinate corresponding to the vertex $w$'s
$i$-th label. We use $x_w$ to indicate the collection of coordinates  corresponding to vertex $w$; i.e., $(x_w^{(1)},
x_w^{(2)},\ldots,x_w^{(k)})$. For a function $f(x):\R^{\dim} \to \R$, we use $f_{u}$ to denote the restriction of $f$ obtained
by setting all the coordinates except $x_u$ to $0$. Similarly, $f_{u,v}$ denotes the restriction of $f$ obtained by setting all
the coordinates except $x_u, x_v$ to $0$.

In the reduction that follows, starting from an instance $\calL$ of Unique Games, we construct a distribution $\calD$ over
labeled examples. Let us denote by $\opt(\calD)$ the agreement of the best degree-$d$ PTF on $\calD$; our constructed
distribution has the following properties:
\begin{itemize}
\item If $\opt(\calL)= 1-\eta$, then $\opt(\calD) = 1-\eta-\frac{1}{\log k}$; and
\item If $\opt(\calL)\leq 1/k^{\theta(\eta)}$, then $\opt(\calD) \leq \half + \frac{2}{\log k}$.
\end{itemize}

This immediately yields the desired hardness result. We now describe and analyze our reduction.

\begin{center}
\fbox{
\parbox{3 in}{
\begin{center}
\textbf{Reduction from Unique Games}
\end{center}

{\bf Input:} Unique Games Instance $\calL(U,V,E,\Pi,k)$.

Set $\beta = \frac{1}{\log k}$ and $\delta= 2^{-k^2}$.

\begin{enumerate}
\item Randomly choose an edge $(u,v) \in E$.
\item Set $y_w = 0$ for any $w\in U \cup V$ such that $w\ne u, w\ne v$.
 \item Generate $k$ i.i.d. bits $a_i \in \{0,1\}$ with $\Pr[a_i=1]=\beta$,
 $2k$ independent standard Gaussians $\{h_i,g_i\}_{i=1}^k$ and a uniform random sign $b\in \{-1,1\}$.
\item For all $i\in [k]$, set $y_v^{(i)}: = g_i $
and $y_u^{(i)} :=  a_i h_i + (g_{\pi^{e}(i)})^d+\delta b$.
\item Output the labeled example $(y,b)$.
\end{enumerate}
}
}
\end{center}

\begin{lemma}[Completeness] \label{lem:cp-red}
If $\opt(\calL)=1-\eta$, then there is a degree-$d$ PTF that is consistent with $1-\eta-\beta$ fraction of the examples.
\end{lemma}

\begin{proof}
Suppose that there is a labeling $L$ that satisfies $1-\eta$ fraction of the edges. Then it is easy to verify that the
degree-$d$ PTF $$\sgn(\littlesum_{u\in U} x_{u}^{(L(u))} - \littlesum_{v\in V} (x_{v}^{(L(v))})^d)$$ agrees with $1-\eta-\beta$
fraction of the examples.
\end{proof}

\begin{lemma}[Soundness] \label{lem:sd-red}
If $\opt(\calL)\leq 1/k^{\Theta(\eta)}$, then no degree-$d$ PTF agrees with more than $1/2+2\beta$
fraction of the examples.
\end{lemma}

\begin{proof}
Suppose (for the sake of contradiction) that some degree-$d$ polynomial $f$ satisfies $1/2+2\beta$ fraction of examples. Then
by an averaging argument, for $\beta$ fraction of the edges $(u,v)$ picked in the first step, we have that $f(x)$ agrees with
the labeled example $(y,b)$ with probability $1/2+\beta$. Let us call these edges``good''.

Fix a ``good'' edge $e=(u,v)$ and let us assume for notational convenience that $\pi^e$ is the identity mapping. Essentially,
we are conducting the test $\calT_d$  for the restriction $f_{u,v}$ with parameter $n:=k$. Since $f_{u,v}$ passes the test with
probability $1/2+ \beta$, Lemma~\ref{lem:soundness} implies that we must have that $I_{0.5}(f_u), I_1(f_v) \neq \emptyset$ and
$|I_{1}(f_v)|,|I_{0.5}(f_u)| \leq 1/\beta^2$.

We are now ready to give our randomized labeling strategy (based on $f$). For every $u\in U$, randomly pick its label from
$I_{0.5}(f_u)$ and for every $v\in V$ randomly pick its label from $I_{1}(f_v)$.  It is clear that each good edge is satisfied
with probability $\beta^2$. Since at least $\beta$ fraction of the edges is good, such a labeling satisfies at least $\beta^3 =
1/(\log k)^3$ fraction of the edges in expectation. Hence, there exists a labeling that satisfies such a fraction of the edges,
which contradicts the assumption that $\opt(\calL)\leq 1/k^{\eta}$, for $k$ sufficiently large.
\end{proof}

\subsection{A technical point:  Discretizing the Gaussian Distribution}\label{sec:samplg}

Lemmas~\ref{lem:cp-red} and~\ref{lem:sd-red} do not quite suffice to prove Theorem~\ref{thm:degreed}, because the reduction described above is not computable in polynomial time.  This is because the distribution ${\calD}$ has infinite support; recall that for each edge $e$, sampling from the corresponding distribution $\calD_e$ requires generating $2k$ independent Gaussian random
variables  $h=(h_1,\ldots,h_k), g=(g_1,\ldots g_k)$.

To discretize the reduction we replace $h$ by $h'$ and $g$ by $g'$, where each of the $2k$ random variables $h_i',g_i'$ is independently generated as a sum of $N$ uniform $\{-1,1\}$ bits divided by $\sqrt{N}$. In Theorem~\ref{thm:closed} of Section~\ref{sec:discg}, we argue that  for sufficiently large $N$ (in particular any $N \geq (2k)^{24(d^2)^2}$ suffices), there is a way to couple the distribution of $(g,h)$ with that of $(g',h')$ such that every degree-$d^2$ polynomial takes the same sign on $(g,h)$ as on $(g',h')$ except with probability at most $1/k.$  Since every outcome of $a \in \{0,1\}^k$ results in the polynomial $f_a(g,h)$ being a
degree-$d^2$ polynomial, if we replace $(g,h)$ with $(g',h')$ in the reduction then the discretized reduction will almost preserve the  soundness and completeness guarantees of
Section~\ref{sec:hrug}, with only a  loss of $\frac{1}{k}$:  writing $\calD'$ for the discretized distribution, we have

\begin{itemize}
\item If $\opt(\calL) \geq 1-\eta$, then $\opt(\calD') \geq 1-\eta-\frac{1}{\log k}-1/k$; and
\item If $\opt(\calL)\leq 1/k^{\eta}$, then $\opt(\calD') \leq \half + \frac{2}{\log k}+1/k$.\
\end{itemize}

Finally, we observe that the distribution of $(g',h')$ has support of size $(N+1)^{2k} \leq (2N)^{2k} \leq
(4k)^{48 d^4  k}$; since the label size $k$ is regarded as constant in a Unique Games instance, this is a (large) constant for constant $d.$  Thus it is possible to simply enumerate the entire support of ${\calD}$ in polynomial time (since there are $|E|$ distributions $\calD_e$, the overall size of the support of ${\calD}$ is polynomial in the size of the Unique Games instance) and consequently there is no need for randomness -- the entire overall reduction is deterministic.
Theorem~\ref{thm:degreed} now follows by choosing appropriate settings of
$\eta$ and $k$ (e.g., $\eta = \eps/2$ and $k=e^{1/\eps^2}$ suffices).

Finally, we note that the above remarks imply that Theorem~\ref{thm:degreed} holds not only
for constant $d$, but for $d$ as large as $O((\log n)^{1/4})$ -- since $k$ is constant,
for such $d$ the support size $(4k)^{48 d^4  k}$ is still polynomial in $n.$

\subsubsection{Discretizing the Gaussian distribution}\label{sec:discg}

The following theorem shows that there exists a distribution $\calH_{N}/\sqrt{N}$ that is point-wise close to a Gaussian distribution $\calG$ with high probability:

\begin{theorem}\label{thm:close}
There is a probability distribution $(\calG,\calH_{N})$ on $\R^2$ such that the marginal distribution
$\calG$ of the first coordinate follows the standard $N(0,1)$ Gaussian distribution, and the marginal distribution $\calH_{N}$ of the second coordinate is distributed as a sum of $N$ random bits, i.e., $\calH_{N} = \sum_{i=1}^N b_i$ where each $b_i$ is an independent random bit from $\{-1,1\}$. In addition, $\calH_{N}$ and $\calG$ are pointwise close in the following sense:
$\Pr[|\calG-\frac{\calH_{N}}{\sqrt{N}}|\leq O(N^{-1/4})] \geq 1-O(N^{-1/4}).$
\end{theorem}

\begin{proof}
Let $\Phi $ be the CDF (cumulative distribution function) of $\calH_{N}$, and let $\Psi$ be the CDF of $\calG$ (the standard Gaussian Distribution).

We couple the random variables $\calG,\calH_{N}$ in the following way:
to obtain a draw $(g_0,h_0)$ from the joint distribution, first we sample $h_0$
from the marginal distribution on $\calH_{N}$.
We know that $$\Pr[\calH_{N}=h_0] = \Phi(h_0)-\Phi(h_0-2),$$  since if $h_0$ is a feasible outcome of summing $N$ bits then $h_0-2$ is the largest feasible outcome that is less than $h_0$ (if any feasible outcome less than $h_0$ exists). Then we generate $g_0$ by drawing random samples from the standard Gaussian distribution until we obtain a sample that lies in the interval
$(\Psi^{-1}(\Phi(h_0-2)),\Psi^{-1}(\Phi(h_0)]$; when we obtain such a sample, we set $g_0$ to this value.

It is not difficult to see that the random variable $\calG$ defined in this way follows the standard Gaussian distribution; essentially we are using the value of $h_0$ as a indicator of whether  $\calG$ is in the interval $(\Psi^{-1}(\Phi(h_0-2)),\Psi^{-1}(\Phi(h_0)]$. We also need to check that $\Pr[\calH=h_0]$ is equal to $\Pr[\calG \in (\Psi^{-1}(\Phi(h_0-2)),\Psi^{-1}(\Phi(h_0))].$
This is true because
\begin{eqnarray*}
&&\Pr[{\cal H} = h_0] \\ &=& \Pr[h\in (h_0-2,h_0]]\\
             &=& \Phi(h_0)-\Phi(h_0-2)\\
             &=& \Pr[\calG \in (\Psi^{-1}(\Phi(h_0-2)),\Psi^{-1}(\Phi(h_0))]].
\end{eqnarray*}

With the above coupling of $\calG$ and $\calH$, it remains to prove that every value in the interval $(\Psi^{-1}(\Phi(h_0-2)),\Psi^{-1}(\Phi(h_0))]$ is close to $h_0/\sqrt{N}$, with high probability over a random choice of $h_0$ as described above.  It suffices to verify that the following two inequalities each hold with probability at least $1 - O(N^{-1/4})$:

\begin{eqnarray*}\left|\Psi^{-1}(\Phi(h_0)) - \frac{h_0}{\sqrt{N}}\right|&\leq& O(N^{-1/4}); \quad \text{and},\\
\left|\Psi^{-1}(\Phi(h_0-2)) - \frac{h_0}{\sqrt{N}}\right|&\leq& O(N^{-1/4}).
\end{eqnarray*}

We consider the first inequality; the first one is entirely similar.  We show that
$\Psi^{-1}(\Phi(h_0)) - \frac{h_0}{\sqrt{N}} \leq O(N^{-1/4})$; the other direction
$\Psi^{-1}(\Phi(h_0)) - \frac{h_0}{\sqrt{N}} \geq -O(N^{-1/4})$ is similar.

By the Berry-Ess{\'e}en Theorem (Theorem~\ref{thm:berry} in Section~\ref{app:prob}),
we have that $|\Phi(h_0) - \Psi(\frac{h_0}{\sqrt{N}})|\leq \frac{1}{\sqrt{N}}.$
Therefore, we have that

\begin{equation}\label{eqn:ttb}\Psi^{-1}(\Phi(h_0))\leq \Psi^{-1}(\Psi({\frac{h_0}{\sqrt{N}}})+\frac{1}{\sqrt{N}})\leq \frac{h_0}{\sqrt{N}} + E_{h_0},
\end{equation}
where the ``error term'' $E_{h_0}$ is the value for which $\Psi(h_0/\sqrt{N} + E_{h_0}) -
\Psi(h_0/\sqrt{N}) = 1/\sqrt{N}.$

If $|h_0| \leq \sqrt{{\frac {N \ln N} 2}}$, then in an interval of width $N^{1/4}$ around
$h_0$ the PDF of the standard Gaussian is everywhere at least $\Omega(N^{-1/4})$; consequently, if $|h_0| \leq \sqrt{{\frac {N \ln N} 2}}$ then the error term $E_{h_0}$ is at most $O(N^{-1/4})$ as required.
%
A standard Chernoff Bound implies that $\Pr[|h_0| < \sqrt{\frac {N \ln N} 2}]$ is at most $O(N^{-1/4})$, and the argument is complete.
%
%
%
\end{proof}

Now we use the joint distribution constructed in Theorem~\ref{thm:close} to discretize the standard
$n$-dimensional Gaussian space for low-degree PTFs.

\begin{theorem}\label{thm:closed}
Fix any constant $D \geq 1$, and let $f(x_1,\ldots,x_n) = \sum_{|S|\leq D}\hat{f}(S) \prod_{i\in S} x_i$
be a degree-$D$ polynomial over $\R^n$.  Let $(y,z)\in \R^{n}\times \R^{n}$ be generated by taking each pair $(y_i,z_i)$ to be an i.i.d. draw from the distribution $(\calG,\calH_{N})$ of Theorem~\ref{thm:close}, where we take $N = n^{24D^{2}}.$
Then we have
$$ \Pr[\sgn(f(y)) \ne \sgn(f(z))] \leq O(1/n).$$
\end{theorem}
\begin{proof}

First, we may assume without loss of generality that the polynomial $f$ is normalized so that
$\sum_{S\ne \emptyset} |\hat{f}(S)|$ equals 1.  Since there are at most ${n + D \choose D}$ coefficients in $f$, one of these coefficients $\hat{f}(S)$ must satisfy $|\hat{f}(S)| \geq {\frac 1 {{n + D \choose D}}}$; now Lemma~\ref{lem:btvar} implies that $\|f\|_2\geq \frac{1}{{n+D\choose D}^2D^D}.$

We have
\begin{eqnarray*}
\Pr[\sgn(f(y)) \ne \sgn(f(z)] &\leq& \Pr[|f(y)|\\
& \leq & |f(z)-f(y)|].
\end{eqnarray*}
To bound the latter probability by $O(1/n)$, we show that $|f(y)| \geq n^{-3D^2}$ with probability
$1-O(1/n)$, and that $|f(z)-f(y)| < n^{-3D^2}$ with probability $1-O(1/n).$

The first desired bound, $\Pr[|f(y)| \leq n^{-3D^2}] \leq O(1/n)$, is an immediate consequence of Theorem~\ref{thm:carberywright}.

For the second, we note that by a union bound and Theorem~\ref{thm:close},  with probability at least
$1-O(n/{N^{1/4}}) \geq 1-O(\frac{1}{n})$ every $i\in [n]$ satisfies $|y_i-z_i|\leq O(N^{-1/4}).$  Standard Chernoff bounds and Gaussian tail bounds give that the probability
any $|y_i|$ or $|z_i|$ exceeds $n^{1/d}$ is much less than $1/n$.
Now similar to the calculation used to bound $f(r+\delta \omega) - f(r)|$ in the proof
of Claim~\ref{claim:small-norm},  when $y$ and $z$ are $O(N^{-1/4})$-close
in each coordinate and each coordinate is at most $n^{1/d}$, we have that
$$|f(y)-f(z)| \leq O(N^{-1/4}) \cdot  O(n) < n^{-3D^2}.$$
This concludes the proof.
%
\end{proof}

\ignore{

\begin{remark} Above theorem immediately implies that $(\frac{\calH_N}{\sqrt{N}})^{\otimes n}$  can be used to fool low degree PTFs over $\calG^{\otimes n}$. Also the distribution of  $\calH_N^{\otimes n}$, by definition, can be generated with $n^{24D^2+1}$ random bits.
\end{remark}

}

\section{Hardness of learning noisy halfspaces with degree $2$ PTF hypotheses:  Proof of
Theorem~\ref{thm:degree2}}\label{sec:degree2}

Similar to Section~\ref{sec:degreed}, the proof has two parts; first (Section~\ref{sec:dictatortest2}) we construct a dictator test for degree $2$ PTFs, and then
(Section~\ref{sec:lc2})
we compose the dictator test with the Label Cover instance to prove NP-hardness.

\subsection{The Dictator Test} \label{sec:dictatortest2}
 The key gadget in the hardness reduction is a Dictator Test that is designed to check whether a degree-$2$ PTF is of the form $\sgn (x_i)$ for some $i\in [n]$.
Suppose $f$ is a degree 2 polynomial
\[
f(x) =\theta + f_1(x) + f_2(x), \quad \text{where}
\]
\[
         f_1(x) = \sum_{i\in [n]} c_i x_i
         \quad \text{and~}f_2(x) =  \sum_{ i,j\in [n], i\leq j} c_{ij} x_i x_j.
\]
Below we give a one-query Dictator Test $\calT_{2}$ for $\sgn(f(x))$.

\bigskip
\fbox{
\parbox{2.8 in}{
\begin{center}
\textbf{$\calT_2$:  Dictator Test for Degree-2 Polynomials  }

\end{center}

\noindent {\bf Input:}  A degree-2 real polynomial $f: \R^n \to \R$

Fix $\beta := \frac{1}{\log n}$ and $\delta := 2^{-n}.$

\begin{enumerate}

\item Generate independent bits $a_1,a_2,\ldots,a_{n} \in \{0,1\}$ each with expected value $\beta.$  Generate $n$ independent $N(0,1)$ Gaussian variables $g_1,\ldots,g_{n}$.  Set $r = (a_1 g_1, a_2 g_2,\ldots,a_n g_n).$

\item  Generate $t$ by randomly picking a number $i\in \{1,2,\ldots,(\log n)^{2}\}$ and set $t = n^i$. Generate a random bit $b\in \{-1,1\}$.
\item Set $\omega\in \R^{n}$ to be the all-1s vector $(1,\ldots,1)$ and set $y =  t^3 r+ b t^2 \delta \omega$.
\item Accept iff $\sgn(f(y)) = b$.
\end{enumerate}
} }

\bigskip

\noindent

We show that $\calT_2$ has the following completeness and soundness properties.
\begin{lemma}{(Completeness)}\label{lem:tcp} For $i \in [n]$, the polynomial $f(x)=x_{i}$ passes
$\calT_2$ with probability at least $1-\beta$.
\end{lemma}
\begin{proof} If $f(x) =  x_i$ for some $i\in [n]$, then as long as $a_i$ is set to zero in step 1 we
have that $f(x) = b \delta t^{2}$ and $f$ passes the test. By definition of the test $a_i$ is 0 with probability $1-\beta$.
\end{proof}

\begin{lemma}{(Soundness)}\label{lem:deg2tsd}
Let $A$ denote $\littlesum_{i=1}^n c_i$ and let $I(f)$ be  the set $ \{i\ | \ c_i > A/n^{2} \}$. If a degree-2
  polynomial $f$ passes the test with probability at least $1/2 + \beta$,  then   $|I(f)|\leq 1/\beta^2$ and $A>0$.
\end{lemma}

\begin{proof}
The proof is by contradiction. Let  $f$ be a degree-2 polynomial with $|I(f)| > 1/\beta^2$ or $A\leq 0$, and suppose that $f$ passes the test with probability at least $\half + \beta$.

First we show the following lemma.
\begin{lemma}
    $\Pr_r[f_1(r) \in (-\delta A, \delta A )] \leq \frac{2}{n}.$
\end{lemma}
\begin{proof}
The inequality obviously holds for $A \leq 0$ since the interval has measure 0.  Thus we may assume that $A> 0$ and $|I(f)|\ge
1/\beta^{2}$. We know that in step 1 when generating the bit-vector $a$,  with probability at least $1 - (1-\beta)^{|I(f)|}
\geq 1-\frac{1}{n}$ at least one of the coordinates in $I(f)$ has its bit $a_i$ nonzero.  Fix any such outcome for the
bit-vector $a$; now considering the random choice of the Gaussians $g_1,\dots,g_n$, we have that the resulting $f_1(r)$ is a
Gaussian variable with variance at least $A^2/n^4$ (as one of the weights is at least $A/n^2$).  Using the standard fact that
an $N(\sigma,\mu)$ Gaussian random variable puts probability mass at most $t/\sigma$ on any interval of length $t$,  we have
that for such an outcome of the $a$-vector,
        \[
                \Pr_{g}[f_1(r) \in (-\delta A, \delta A)]\leq \frac{2\delta A} {A/n^2} \leq \frac{n^3}{2^n} \leq \frac{1}{n}.
        \]

Now a union bound gives that  for at most $\frac{2}{n}$ of the $r$ generated, $f(r)$ is  inside the interval $(-\delta A, \delta
A)$.
\end{proof}

Now we observe that for any outcome $r$, the vectors $r$ and $-r$ are generated with equal probability.  Thus an equivalent test to $\calT_2$ would be to generate $r,t$ as described by the test and then check a randomly selected one of the following four inequalities:
\begin{eqnarray}
\label{eq:a} f(t^{3}r+t^{2}\delta \omega )  \geq 0 \\
\label{eq:b} f(t^{3}r-t^{2}\delta \omega) < 0\\
\label{eq:c} f(-t^{3}r +t^{2} \delta \omega ) \geq 0 \\
\label{eq:d} f(-t^{3}r - t^{2} \delta \omega )  < 0.
 \end{eqnarray}

 Since $f$ is assumed to pass the test with probability $\half + \beta$ an averaging argument gives that for a $\beta/2$ fraction of the possible outcomes of $r$, at least a $(\half
+\beta/2)$ fraction of all the constraints involving that $r$ outcome are satisfied. (Note that for any fixed outcome of $r$ there are
$4(\log n)^2$ constraints, corresponding to inequalities (\ref{eq:a})--(\ref{eq:d}) for each of the $(\log n)^2$ possible
values of $t.$)  For this $\beta/2$ fraction of $r$, let us  remove those outcomes $r$ such that
$p_1(r) \in (-\delta A , \delta A)$ (recall that this is at most a $2/n$ fraction of all $r$-outcomes).  Recalling that $\beta = \frac{1}{\log n}$, we know there are at least $\beta/4$ fraction of
$r$-outcomes remaining; we call these ``good'' $r$'s.

Let us  fix a good  $r$. By an averaging argument again, for any  ``good'' $r$, for at least a $ \beta/4$ fraction of the possible outcomes of $t$, at least $3$ out of the $4$ of the inequalities that contain $t$ and $r$ are satisfied. There are $4$
different ways of choosing $3$ out of the $4$ constraints. Without loss of generality, let us assume that for a $\beta /16$
fraction of the $t$-outcomes, the first, second, and fourth constraints (\ref{eq:a}), (\ref{eq:b}) and (\ref{eq:d}) are satisfied. That
is:
\begin{gather}
\label{eqn1}f(t^{3}r+t^{2}\delta \omega )  > 0 \\
 \label{eqn2}f(t^{3}r-t^{2}\delta \omega) < 0\\
\label{eqn3} f(-t^{3}r -t^{2} \delta \omega ) < 0.
 \end{gather}
Let us call these  $t$ ``good" for the corresponding $r$, and let us denote the set that contains all the ``good'' $t$ for a given
``good'' $r$ by  $T_r$. Since the possible choice of $t=n^i$ ranges over all $i \in [\log^2 n]$, we therefore obtain $|T_r|\geq (\log n)^{2}\cdot \beta/16 = \Theta(\log n)$.

Since $f(x)$  is a degree $2$ polynomial, we can express $f(r+\delta \omega)$ as:
\begin{multline*}
 f(r+\delta \omega )  =\theta + f_1(r) +f_2(r) \\ + \delta \sum_{i=1}^n c_{i} + \delta^{2} \sum_{1 \leq i \leq j \leq n} c_{ij}  + \delta \sum_{1\leq i \leq j \leq n} c_{ij} (r_i+r_j).
\end{multline*}

Let us denote $B = \sum_{1\leq i \leq j\leq n} c_{ij}$ and $f'_2(r) = \sum_{1\leq i \leq j\leq n} c_{ij} (r_i+ r_j)$. We can rewrite \eqref{eqn1},
\eqref{eqn2}, \eqref{eqn3} as:
\begin{gather}
\label{eqn:org1}        t^3 f_1(r) + t^2\delta A+ t^6 f_2(r) + t^5\delta f'_2(r) + t^4 \delta^2 B +\theta \geq 0 \\
\label{eqn:org2}        t^3 f_1(r) - t^2\delta A+ t^6 f_2(r) - t^5\delta f'_2(r) + t^4 \delta^2 B + \theta  < 0 \\
\label{eqn:org3}        t^3 f_1(r) + t^2\delta A-t^6 f_2(r) - t^5\delta f'_2(r) - t^4 \delta^2 B -\theta > 0
\end{gather}

Notice that \eqref{eqn:org1} and \eqref{eqn:org3} yield
\[
        f_1(r) \geq -\delta A / t + \big| t^3 f_2(r)+\delta t^2 f'_2(r)+  \delta^2 t B + \theta/t^3 \big|.
\]

Since we already know that $f_1(r)\notin (-\delta A , \delta A)$ and $t$ is at least 1, we get that
\[
 f_1(r) \geq \delta A.
\]

Also for \eqref{eqn:org2}, we can rewrite it as
\[
        f_1(r) \leq \delta A /t - (t^3 f_2(r) - \delta t^2 f'_2(r) + \delta^2 t  B + \theta /t^3).
\]

Let us further simplify the notation by writing $C$ for
$f_2(r)$, $D$ for $\delta f_2'(r)$ and $E$ for $\delta^2 B$. Then we may rewrite the
above constraints as follows:
\[
        f_1(r) \geq -\delta A / t + \big| t^3 C + t^2 D+  t E + \theta/t^3 \big|
\]
and
\begin{equation}
\label{eq:star}      \delta A \leq  f_1(r) \leq \delta A /t - (t^3 C -  t^2 D + t  E + \theta /t^3).
\end{equation}

Notice that above (upper and lower) bound hold for any $t$ in $T_r$. Therefore, we know that for any $t_1,t_2\in T_r$,
\begin{multline*}
        \delta A /t_1 - (t_1^3 C -  t_1^2 D +  t_1  E + \theta /t_1^3) \\ \geq  -\delta A / t_2 + \big| t_2^3 C + t_2^2 D+  t_2 E + \theta/t_2^3 \big|
\end{multline*} which  is equivalent to
\begin{multline}\label{eqn:ul}
         - (t_1^3 C -  t_1^2 D +  t_1  E + \theta /t_1^3) + \delta A (\frac{1}{t_1} + \frac{1}{t_2}) \\ \geq  \big| t_2^3 C + t_2^2 D+  t_2 E + \theta/t_2^3 \big|.
\end{multline}

Using the fact that $f_1(r) > \delta A$, the inequality (\ref{eq:star}) gives  $- (t_1^3 C -  t_1^2 D + t_1  E + \theta /t_1^3)
> (1-\frac{1}{t_1})\delta A,$ which may be rewritten as $\delta A \leq {\frac {- (t_1^3 C -  t_1^2 D + t_1  E + \theta
/t_1^3)}{1-1/t_1}}.$   Combining this with \eqref{eqn:ul}, we know that for any $t_1,t_2 \in T_r$, we have
\begin{multline*}
- (t_1^3 C -  t_1^2 D + t_1 E + \theta /t_1^3)  \left(1 + \frac{(\frac{1}{t_1} + \frac{1}{t_2})}{1-\frac{1}{t_1}}\right) \\ \geq \big| t_2^3
C + t_2^2 D+  t_2 E + \theta/t_2^3 \big|.
\end{multline*}
By definition, $t_i \geq n$ for any $i$, so we have $\frac{(\frac{1}{t_1} + \frac{1}{t_2})}{1-\frac{1}{t_1}}\leq 3/n$.
Therefore, for any $t_1,t_2$ in $T_r$, the following inequality holds:
\begin{equation}\label{eqn:big}
\frac{- (t_1^3 C +  t_1^2 D-  t_1  E + \theta /t_1^3)}{\big| t_2^3 C + t_2^2 D+  t_2 E + \theta/t_2^3 \big|} \geq \frac{1}{1 +
\frac{(\frac{1}{t_1} + \frac{1}{t_2})}{1-\frac{1}{t_1}}  } \geq 1-3/n.
\end{equation}
Note that the denominator of the LHS of (\ref{eqn:big}) can be zero for at most 6 values of $t_2$; we eliminate any such values from $T_r$, and we still have $|T_r| \geq \Theta(\log n)$.
(Actually, we
will only need $|T_r| \geq 5$ for the remainder of the argument to establish the required contradiction.)  Let us pick  $t_0 < t_1 < t_2 < t_3 < t_4$ from  $T_r$, and let us write $G$ to denote $- (t_1^3 C -
t_1^2 D + t_1  E + \theta /t_1^3)$. We know that \[G \leq t_1^3|C| +  t_1^2 |D| + t_1  |E| + |\theta| /t_1^3. \]

Also for $t_0,t_2,t_3,t_4$, we write:
\begin{gather}
\label{eqn:1}        F_0 := t_0^3 C - t_0^2 D+  t_0 E + \theta/t_0^3  \\
\label{eqn:2}        F_2 := t_2^3 C - t_2^2 D+  t_2 E + \theta/t_2^3 \\
\label{eqn:3}        F_3 := t_3^3 C - t_3^2 D+  t_3 E + \theta/t_3^3 \\
\label{eqn:4}        F_4 := t_4^3 C - t_4^2 D+  t_4 E + \theta/t_4^3.
\end{gather}
Let $F$ denote $\max_{i=0,2,3,4} |F_i|$. By \eqref{eqn:big} we know that
\begin{equation}\label{eqn:big1}
\frac{G}{F}\geq 1-3/n.
\end{equation}
Viewing $C,D,E,\theta$ as unknowns, we may solve the above linear system consisting of equations
\eqref{eqn:1},\eqref{eqn:2},\eqref{eqn:3},\eqref{eqn:4} using Cramer's rule.  We find that
\[
        C=\frac{\left|{\begin{array}{cccc}
F_0 & -t_{0}^2 & t_0 & 1/t_0^3 \\
F_2 & -t_{2}^2 & t_2 & 1/t_2^3\\
F_3 & -t_{3}^2 & t_3 & 1/t_3^3 \\
F_4 & -t_{4}^2 & t_4 & 1/t_4^3 \\
\end{array}}\right|}{\left|\begin{array}{cccc}
t_{0}^3 & -t_{0}^2 & t_0 & 1/t_0^3 \\
t_{2}^3 & -t_{2}^2 & t_2 & 1/t_2^3\\
t_{3}^3 & -t_{3}^2 & t_3 & 1/t_3^3 \\
t_{4}^3 & -t_{4}^2 & t_4 & 1/t_4^3 \\
\end{array}\right|} =
\frac{\left|{\begin{array}{cccc}
F_0 & t_{0}^2 & t_0 & 1/t_0^3 \\
F_2 & t_{2}^2 & t_2 & 1/t_2^3\\
F_3 & t_{3}^2 & t_3 & 1/t_3^3 \\
F_4 & t_{4}^2 & t_4 & 1/t_4^3 \\
\end{array}}\right|}{\left|\begin{array}{cccc}
t_{0}^3 & t_{0}^2 & t_0 & 1/t_0^3 \\
t_{2}^3 & t_{2}^2 & t_2 & 1/t_2^3\\
t_{3}^3 & t_{3}^2 & t_3 & 1/t_3^3 \\
t_{4}^3 & t_{4}^2 & t_4 & 1/t_4^3 \\
\end{array}\right|}
.
\]
Since $0 < t_0<t_2<t_3<t_4$ and these values are at least a factor of $n$ apart from each other, we have that
\[
\left|\begin{array}{cccc}
t_{0}^3 & t_{0}^2 & t_0 & 1/t_0^3 \\
t_{2}^3 & t_{2}^2 & t_2 & 1/t_2^3\\
t_{3}^3 & t_{3}^2 & t_3 & 1/t_3^3 \\
t_{4}^3 & t_{4}^2 & t_4 & 1/t_4^3 \\
\end{array}\right|\]
is $\Omega(t_4^3t_3^2t_2t_0^{-3})$.

Since $F = \max_{i=0,2,3,4} |F_i|$, we know that the absolute value of
\[
\left|{\begin{array}{cccc}
F_0 & t_{0}^2 & t_0 & 1/t_0^3 \\
F_2 & t_{2}^2 & t_2 & 1/t_2^3\\
F_3 & t_{3}^2 & t_3 & 1/t_3^3 \\
F_4 & t_{4}^2 & t_4 & 1/t_4^3 \\
\end{array}}\right|
\]
is at most $ O(F t_4^2t_3t_0^{-3}).$
Thus we have  $|C| = O(\frac{F}{t_4t_3t_2}).$

Similar analysis shows that
\[      |D| = O({F}/{t_3t_2}); \quad
        |E| = O({ F}/{t_2}); \text{and }
        |\theta| =  O(F t_0^3).
\]

Therefore, we have
\begin{multline*}
G \leq |C| t_1^3 + t_1^2 |D| + t_1 |E| + |\theta|/ t_1^3
 \\ \leq F \cdot O(t_1^3/t_4t_3t_2+ + t_1^2/t_2t_3 + t_1/t_2 + t_0^3/t_1^3).
\end{multline*}
Recalling that $t_{i+1}/t_i\geq n$ as they are different powers of $n$, we have that
\[
\frac{G}{F}     \leq O({1}/{n}).
\]
This contradicts \eqref{eqn:big1} and concludes the proof of the soundness Lemma,
Lemma~\ref{lem:deg2tsd}.
\end{proof}

\subsection{Hardness reduction from Label Cover} \label{sec:lc2}

Recall that our reduction is from a Label Cover instance $\calL$ specified by $(U,V,E,k,m,\Pi)$ . For notational convenience let us write $F(q)$ to denote the space of possible labels for vertex $q \in U\cup V$, for $u\in U$, $F(u) $ denotes $[k]$ and for $v\in V$, $F(v)$ denotes $[m]$.

We reduce to a learning problem with labeled examples in  $\R^{|U|k+ |V|m} \times \{-1,1\}.$  Let $\dim$ denote
$|U| k+ |V| m$. For $y\in \R^{\dim}$ and $q\in U\cup V$, we write $y_{q}^{(i)}$ to denote the vector consisting of all coordinates that correspond to vertex $q$, i.e. $y_u$ denotes $(y^{(i)}_u)_{i \in [k]}$ for $u \in U$ and
$y_{v}$ denotes $(y^{(i)}_v)_{i \in [m]}$ for $v \in V.$

We give the reduction from Label Cover to the learning problem below.  The high level idea is that the Dictator Test $\calT_2$ is performed on the restricted function $p_v(y)$ for a random $v\in V$.


\begin{center}
\fbox{
\parbox{3 in}{
\begin{center}
\textbf{Reduction from Label-Cover $\calL$}
\end{center}

{\bf Input:} Label Cover Instance $(U,V,E,k,m,\Pi)$.

\begin{enumerate}
\item Randomly pick a vertex $v \in V$.
\item  For each $w \neq v,$ $w \in U\cup V$, set $y_w = 0$.
\item Let $a_1,\dots,a_m$ be independent $\{0,1\}$ bits each with $\E[a_i] = \beta.$  Let $g_1,\dots,g_m$ be
 independent $N(0,1)$ Gaussian random variables.  Let $i$ be chosen uniformly from $[(\log m)^{2}]$ and set $t = m^i$.  Let $b$ be a random uniform bit from $\{-1,1\}$.
\item Set $r = (a_1 g_1, a_2 g_2,\ldots,a_m g_m).$
\item Let  $\omega \in \R^{m}$ be $\omega = (1,\dots,1)$, and set $y_v:=  t^3 r+ b t^2 \delta \omega$.

\item Output the labeled example $(\Fold(y_v),b)$  (we describe the folding procedure $\Fold(\cdot)$ later).
\end{enumerate}
} }
\end{center}
The learning problem is to find a degree 2 polynomial $p: \R^{dim}\to \{-1,1\}$ such that $\sgn(p(y)) = b$ for the largest possible
fraction of labeled examples generated as described above. Let us denote
\begin{multline*}
p(y) = \theta + \sum_{q\in U\cup V, i \in F(q)} c_{q}^{(i)} y_{q}^{(i)} \\ + \sum_{q_{1},q_{2}\in U\cup V, i \in F(q_{1}, j\in F(q_{2})} c_{(q_{1},q_{2})}^{(i,j)} y_{q_{1}}^{(i)}y_{q_{2}}^{(j)}.
\end{multline*}

Notice that in the reduction, when vertex $v$ is picked we set all the coordinates to zero except $y_v$. Essentially  we are performing the test $\calT_{2}$ on  the function $$p_v = \theta + \sum_{i\in [m]} c_{v}^{(i)} y_{v}^{(i)} + \sum_{i,j \in [m]}
c_{(v(i),v(j))} y_{v}^{(i)}y_{v}^{(j)} $$ which is the restriction of $p(y)$ obtained by setting all the coordinates to zero except those coordinates corresponding to vertex $v$.  The overall fraction of agreement  of  $p(y)$ on all   examples is the average probability, over all $v \in V$, that $p_v$ passes $\calT_{2}$.

\paragraph{Folding Trick:} We use the ``folding ''
technique that was first introduced in \cite{GKS10journal,KS08}.  The trick essentially amounts to the following: instead of outputting the labeled example
$(y,b)$ in the last step of the reduction, we output $(\Fold(y),b)$ where $\Fold(y)$ is the projection of $y$ into a subspace $H^{\perp}$ (defined below). Folding enables us to enforce that $p$ takes the same value on different points in $\R^{\dim}$ as long as they project to the same point in $H^{\perp}$.

We define the subspaces $H,H^{\perp}$ for our folding as follows:
\begin{definition}
For every  $e = \{u,v\}\in E, i \in [k]$, we define $b(e,i)\in \R^{\dim}$ to be the vector that has $0$ at every coordinate except that
${b(e,i)}_{u}^{(i)}$ is $1$ and for every  $j\in (\pi^{e})^{-1}(i)$, $b(e,j)_{v}^{(j)}$ is $-1$.  Let $B$ be the collection of
all such $b(e,i)$, i.e.   $B = \{b(e,i)\ |\ e = \{u,v\}\in E, i \in [k]\}$. We define $H$ to be $span(B)$ and
$H^{\perp}$ to be the orthogonal complement of $H$ in $R^{dim}$.
\end{definition}

We define $\Fold(y)$ to be the projection of $y$ onto $H^\perp.$  It is easy to see that the mapping $\Fold(\cdot)$ can be performed in polynomial time.

After the folding  procedure, we can further enforce $p(x)$ to have the  property:
\[
\text{For any $h\in H$ and } x\in \R^{\dim}, p(x+h) = p(x).
\]
We call functions that have the above property ``folded''. In particular for  $e = \{u,v\}\in E$, $c \in \R$, and $i\in [k]$, a folded function $p$ satisfies $p(x + c b(e,i)) =
p(x)$. If we view $p(y)$  as a polynomial only on $y_{u}^{(i)}$ and $y_{v}^{(j)}$ for $j\in (\pi^{e})^{-1}(i)$, then Lemma~\ref{lem:fd} shows that we have the
following folding property of $p$:  $$c_{u}^{(i)} = \sum_{j\in (\pi^e)^{-1}}c_{v}^{(j)}.$$

If we sum over all possible $i$, this implies for any edge $\{u,v\}$, we have
$$\sum_{i \in [k]} c_{u}^{(i)} = \sum_{i\in [m]} c_{v}^{(i)}.$$
\ignore{Without loss of generality, we can assume for any $q\in U\cup V$, $$\sum_{i \in s[q]} c_{q}^{(i)}$$ is the same.
\ynote{do we need to do addintional folding or assuming the graph is connected}
}

\noindent Now we are ready to prove Theorem~\ref{thm:degree2}.
We  will show  the following two properties of the reduction to complete the proof.


\begin{lemma}[Completeness]\label{thm:cp}
If $ \opt(\calL) = 1$, then there is a  folded function $p(x)$ that is consistent with $1-1/\log m$ fraction of the labeled
examples generated by the reduction.
\end{lemma}
\begin{lemma}[Soundness] \label{thm:sd} If $\opt(\calL)\leq 1/m^{\eta}$, then there is no  folded degree-$2$ polynomial
that is consistent with $1/2+\frac{2}{\log^{2}m}$ fraction of the labeled examples generated by the reduction.
\end{lemma}
Combining Lemmas~\ref{thm:cp} and \ref{thm:sd} and noticing that $m$ can be an arbitrarily large constant (such as
$e^{1/\eps^2}$),we obtain Theorem~\ref{thm:degree2}.  (A discretization similar to that of Section~\ref{sec:samplg} is also
required, and can be obtained in a routine way by slightly modifying the parameters of that section's construction.)


\paragraph{Proof of Theorem~\ref{thm:cp}:}   Suppose that $\opt(\calL) = 1$, so there is a labeling $l$ satisfying all the edges.  Then consider the following function   $$p(x) = \sum_{w\in U\cup V} x_{w}^{(l(w))}.$$  For every $v\in V$, the function $p_{v}$ is a dictator and passes $\calT_{m}$ with probability at least $1-\frac{1}{\log m}$ by Lemma~\ref{lem:tcp}. Consequently the overall probability that $p$ passes the test is at least
$1 - 1/\log m$.  Finally, it is easy to check that thus function $p(x)$ is folded. \qed

\paragraph{Proof of Theorem~\ref{thm:sd}:}
Suppose that there is some folded degree-2 polynomial $p(x)$ such that $\sgn(p(x))$ agrees with more
than $\half + \frac{2}{\log m}$  fraction of the example, i.e., the averaging passing probability of $p_v$ on $\calT_m$ is
$\half + \frac{2}{\log m}$. We will show that $\opt(\calL) > 1/m^\eta$ and thus prove the theorem.

By an averaging argument,  we know for a $\frac{1}{\log m}$ fraction of the vertices $v\in V$, the restricted polynomial $p_v$ passes the test $\calT_{k}$  with probability at laest $\half+\frac{1}{\log m}$; we refer to any such $v$ as a ``good'' vertex.  We say that an edge is ``good'' if the $V$-endpoint of the edge is a good vertex. Since the graph is regular, we know that at least a $\frac{1}{\log m}$ fraction of all edges are ``good''.

For a ``good'' vertex $v$, let us define $I_v$ to be $$I_v = \{j | \  j\in [m], c_{v}^{(j)}
> \sum_{i=1}^{m} c_{v}^{(i)}/m^{2}\}.$$ By Lemma~\ref{lem:deg2tsd}, we have $|I_{v}| \leq (\log m)^{2}$ and  $\sum_{i\in [m]} c_{v}^{(i)}> 0$.
For every $u\in  U$, we define $J_u = \{j | \ j\in [k], c_{u}^{(j)} \geq \sum_{i\in[k]} c_{u}^{(i)}/k\}$.
We note that $J_u$ is not empty as $$ \max_{j} c_{u}^{(j)}\geq \sum_{i\in[k]} c_{u[i]}/k.$$

We define the following labeling strategy for $\calL$. For $u\in U$, randomly assign it a label from $J_u$;  for $v\in V$, randomly assign it a label from $I_v$ (if $I_v$ is empty, we assign a random label to $v$).

For every good edge $e=(u,v)$ and any $j\in J_u$, since $p$ is folded, we have that
\[
       \sum_{i\in\pi_e^{-1}(j)}c_{v}^{(i)}=c_{u}^{(j)} \geq \sum_{i\in [k]} c_{u}^{(i)} / k= \sum_{i\in [m]} c_{v}^{(i)}/k.
\]
There is at least one label $i$ in $\pi_e^{-1}(j)$ such that $ \sum_{i\in [m]} c_{v}^{(i)}/km  \geq  \sum_{i\in [m]}
c_{v}^{(i)}/m^{2} $, and this label is therefore in $I_v$. As noted earlier we have  $|I_{v}| \leq (\log m)^{2}$, and so by our randomized labeling strategy there is at least a $1/(\log m)^{2}$ probability that edge $\{u,v\}$ is satisfied.

Therefore the above labeling strategy satisfies (in expectation) at least $1 /(\log m)^2$ fraction of the good edges and consequently at least $ 1/(\log m)^{3} $ fraction of all edges.  This means that $\opt(\calL) > 1/m^\eta$ and the proof is complete. \qed

\subsubsection{Folding Lemma}

\begin{lemma} \label{lem:fd}
Let
$$f(x) = \theta + \sum_{i=0}^{n}   w_i x_i + \sum_{0\leq
i\leq j\leq n}w_{ij}x_ix_j$$
be a degree 2 function. Suppose that for every $x \in \R^n, c \in \R$ we have $f(x+c(1,-1,\ldots,-1)) = f(x)$.
Then $w_0 = \sum^n_{i=1} w_i$.

\end{lemma}
\begin{proof}
Expanding the equality $f(x+c(1,-1,\dots,-1)) = f(x)$, we get that
\begin{multline*}
        \theta + w_0 (x_0+c) + \sum_{i=1}^n w_i(x_i-c)+w_{00}(x_0+c)^2 \\+ \sum_{j=1}^n w_{0j}
        (x_0+c)(x_j-c)+ \sum_{1\leq i\leq j\leq n} w_{ij}(x_i- c)(x_j-c) \\=  \theta + \sum_{i=0}^{n}   w_i x_i + \sum_{0\leq i\leq j\leq n}w_{ij}x_ix_j.
\end{multline*}
Since this equation holds for all $c,x$, if we express the LHS and RHS as  polynomials in the variables $c, x_0, x_1,\ldots,x_n$, the corresponding coefficients must be the same.
If we look at the coefficients of the degree-1 monomial $c$, we have that  $w_0 - \sum_{i=1}^n w_i= 0$, and the lemma is proved.
\end{proof}
\section{Conclusion}

We have established two hardness results for proper agnostic learning of low-degree PTFs.  Our results show that even if there exist low-degree PTFs that are almost perfect hypotheses, it is computationally
hard to find low-degree PTF hypotheses that perform even slightly better than random guessing; in this
sense our hardness are rather strong.  However, our results do not rule out the possibility of efficient learning algorithms when $\eps$ is sub-constant, or if unrestricted hypotheses may be used.  Strengthening the hardness results along these lines is an important goal for future work, but may require significantly new ideas.

Another natural goal for future work is the following technical strengthening of our results:
show that for any constant $d$, it is hard to construct a degree-$d$ PTF that is consistent with $(\half + \eps)$ fraction of a given set of labeled examples, even if there exists a halfspace that is consistent with a $1-\eps$ fraction of the data.  Such a hardness result would subsume both of the results of this paper as well as much prior work, and would serve
as strong evidence that agnostically learning halfspaces under arbitrary distributions
is a computationally hard problem.
\appendix
\section*{Appendix}
\bigskip
\section{Probability inequalities}\label{app:prob}

We will use the Berry-Ess{\'e}en Theorem, which is a quantitative version of the Central Limit Theorem:

\begin{theorem} (Berry-Ess{\'e}en Theorem)\label{thm:berry}
Let $x_1,x_2,\ldots,x_n$ be i.i.d. uniform $\{-1,1\}$-valued random variables.  Let $c_1,\dots,c_n
\in \R$ be such that $\sum_{i=1}^{n}c_i^2 = 1$ and $\max_{i} |c_i|\leq \tau$. Let $g$ denote a unit Gaussian variable drawn from $N(0,1)$. Then for any $\theta \in \R$, we have
\[
        |\Pr[\sum_{i=1}^n c_{i} x_{i} \leq \theta] - \Pr[g\leq \theta]| \leq \tau.
\]
\end{theorem}

We will also use the following anti-concentration result
for low-degree polynomials over Gaussian random variables, due to Carbery
and Wright:

\begin{theorem}[\cite{CW01}]\label{thm:carberywright}
Let $p: \R^n \to \R$ be a nonzero degree-$d$ polynomial over the reals. Then for all $\tau>0$, we have
\[ \Pr_{x \sim \N^n}[|p(x)|\leq \tau \|p\|_2 ]\leq O(d\tau^{1/d}) .\]
\end{theorem}

\bibliographystyle{abbrv}
\bibliography{learning,everything}

\end{document}